\documentclass[sigconf,nonacm]{acmart}
\AtBeginDocument{%
  \providecommand\BibTeX{{%
    \normalfont B\kern-0.5em{\scshape i\kern-0.25em b}\kern-0.8em\TeX}}}

\setcopyright{acmcopyright}
\copyrightyear{2022}
\acmYear{2022}
\acmDOI{10.1145/1122445.1122456}

\acmConference[SIGMOD '22]{SIGMOD '22: International Conference on Management of Data}{June 12--17, 2022}{Philadelphia, USA}

\usepackage{enumitem}
\setitemize{noitemsep,topsep=0pt,parsep=0pt,partopsep=0pt}
\setenumerate{noitemsep,topsep=0pt,parsep=0pt,partopsep=0pt}

\usepackage{color,soul}
\usepackage{listings}
\usepackage{lipsum}

\usepackage{amsmath}
\usepackage{empheq}
\usepackage[most]{tcolorbox}
\usepackage{subfigure}
\usepackage{graphicx}

\usepackage{algorithm}
\usepackage{algorithmic}

\usepackage{xspace}
\newcommand{\sys}{\textsc{Persia}\xspace}
\newcommand{\bagua}{\textsc{Bagua}\xspace}

\newcommand{\kuaishou}{Kwai\xspace}

\newcommand{\assign}{:=}

\newcommand{\mathD}{\mathrm{D}}
\newcommand{\embWeight}{\mathbf{w}^{\mathrm{emb}}}
\newcommand{\embDimension}{N^{\mathrm{emb}}}
\newcommand{\nnDimension}{N^{\mathrm{nn}}}
\newcommand{\embGradient}{{F_{\xi}^{\mathrm{emb}}}'}
\newcommand{\embWeightIter}[1]{\mathbf{w}^{\mathrm{emb}}_{\mathrm{#1}}}
\newcommand{\embGradientIter}[1]{{{F^{\mathrm{emb}}_{#1}}}'}
\newcommand{\idInput}{\mathbf{x}^{\mathrm{ID}}}
\newcommand{\idInputIndexed}[1]{{\mathbf{x}_{#1}^{\mathrm{ID}}}}

\newcommand{\nnWeight}{\mathbf{w}^{\mathrm{nn}}}
\newcommand{\nnGradient}{{F_{\xi}^{\mathrm{nn}}}'}
\newcommand{\nnWeightIter}[1]{\mathbf{w}^{\mathrm{nn}}_{\mathrm{#1}}}
\newcommand{\nnGradientIter}[1]{{{F^{\mathrm{nn}}_{#1}}}'}
\newcommand{\nidInput}{\mathbf{x}^{\mathrm{NID}}}

\newcommand{\nidInputIndexed}[1]{{\mathbf{x}_{#1}^{\mathrm{NID}}}}

\newcommand{\lookup}{\mathrm{lookup}_{\embWeight}}
\newcommand{\nn}{\mathrm{NN}_{\nnWeight}}

\newcommand{\w}{\mathbf{w}}
\newcommand{\f}[1]{f\left(#1\right)}
\newcommand{\F}[1]{F\left(#1\right)}
\newcommand{\expect}[2]{\mathbb{E}_{#1}\left[#2\right]}
\newtheorem{assumption}{Assumption}
\newtheorem{theorem}{Theorem}
\newtheorem{lemma}{Lemma}
\newtheorem{remark}{Remark}
\newcommand\numberthis{\addtocounter{equation}{1}\tag{\theequation}}

\begin{document}

\title{\sys:  An Open, Hybrid System Scaling Deep Learning-based Recommenders up to 100 Trillion Parameters}

\author{Xiangru Lian$^1$, Binhang Yuan$^3$, Xuefeng Zhu$^2$, Yulong Wang$^2$, Yongjun He$^3$, Honghuan Wu$^2$, Lei Sun$^2$, Haodong Lyu$^2$, Chengjun Liu$^2$, Xing Dong$^2$, Yiqiao Liao$^2$, Mingnan Luo$^2$, Congfei Zhang$^2$, Jingru Xie$^2$, Haonan Li$^2$, Lei Chen$^2$, Renjie Huang$^2$, Jianying Lin$^2$, Chengchun Shu$^2$, Xuezhong Qiu$^2$, Zhishan Liu$^2$, Dongying Kong$^2$, Lei Yuan$^2$, Hai Yu$^2$, Sen Yang$^2$, Ce Zhang$^3$, Ji Liu$^1$}
\affiliation{$^1$Kwai Inc. \country{USA}; $^2$Kuaishou Technology, China; $^3$ETH Z\"urich, Switzerland;}
\email{{firstname.lastname}@{1.kwai.com; 2.kuaishou.com; 3.inf.ethz.ch }}

\renewcommand{\shortauthors}{Lian et al.}

\begin{abstract}

Deep learning based models have dominated the current landscape of production recommender systems.
Furthermore, recent years have witnessed an exponential growth of the model scale---from Google's 2016 model with 1 billion parameters to the latest Facebook's model with 12 trillion parameters.
Significant quality boost has come with each jump of the model capacity, which makes us believe the era of 100 trillion parameters is around the corner. However, the training of such models is challenging even within industrial scale data centers. This difficulty is inherited from the staggering heterogeneity of the training computation---the model's embedding layer could include more than $99.99\%$ of the total model size, which is extremely memory-intensive; while the rest neural network is increasingly computation-intensive. 
To support the training of such huge models, an efficient distributed training system is in urgent need.
In this paper, we resolve this challenge by careful co-design of both the optimization algorithm and the distributed system architecture.
Specifically, in order to ensure both the training efficiency and the training accuracy, we design a novel hybrid training algorithm, where the embedding layer and the dense neural network are handled by different synchronization mechanisms; then we build a system called \sys (short for \textbf{p}arallel r\textbf{e}commendation t\textbf{r}aining \textbf{s}ystem with hybr\textbf{i}d
\textbf{a}cceleration) to support this hybrid training algorithm.
Both theoretical demonstrations and empirical studies up to 100 trillion parameters have been conducted to justified the system design and implementation of \sys.
We make \sys publicly available (at \url{https://github.com/PersiaML/Persia}) so that anyone would be able to easily train a recommender model at the scale of 100 trillion parameters.

\end{abstract}

\maketitle

\section{Introduction}
\label{sec:intro}

A recommender system is an important component of Internet services today. Tasks such as click-through rate (CTR) and buy-through rate (BTR) predictions
are widely adopted in industrial applications, influencing the ad revenues at billions of dollar level for search engines such as Google, Bing and Baidu~\cite{su2017improving}. Moreover, $80\%$ of movies watched on Netflix~\cite{gomez2015netflix} and $60\%$ of videos clicked on YouTube~\cite{davidson2010youtube} are driven by
automatic recommendations; over $40\%$ of user engagement on Pinterest are powered by its Related Pins recommendation module~\cite{liu2017related}; over half of the Instagram community has visited recommendation based Instagram Explore to discover new
content relevant to their interests~\cite{acun2020understanding}; up to $35\%$ of Amazon’s revenue is driven by recommender systems~\cite{xie2018personalized,chang2021extreme}.
At \kuaishou, we also observe that recommendation plays an important role for video sharing---more than 300 million of daily active users explore videos selected by recommender systems from billions of candidates.

\begin{figure}[t!]
    \centering\includegraphics[width=.5\textwidth]{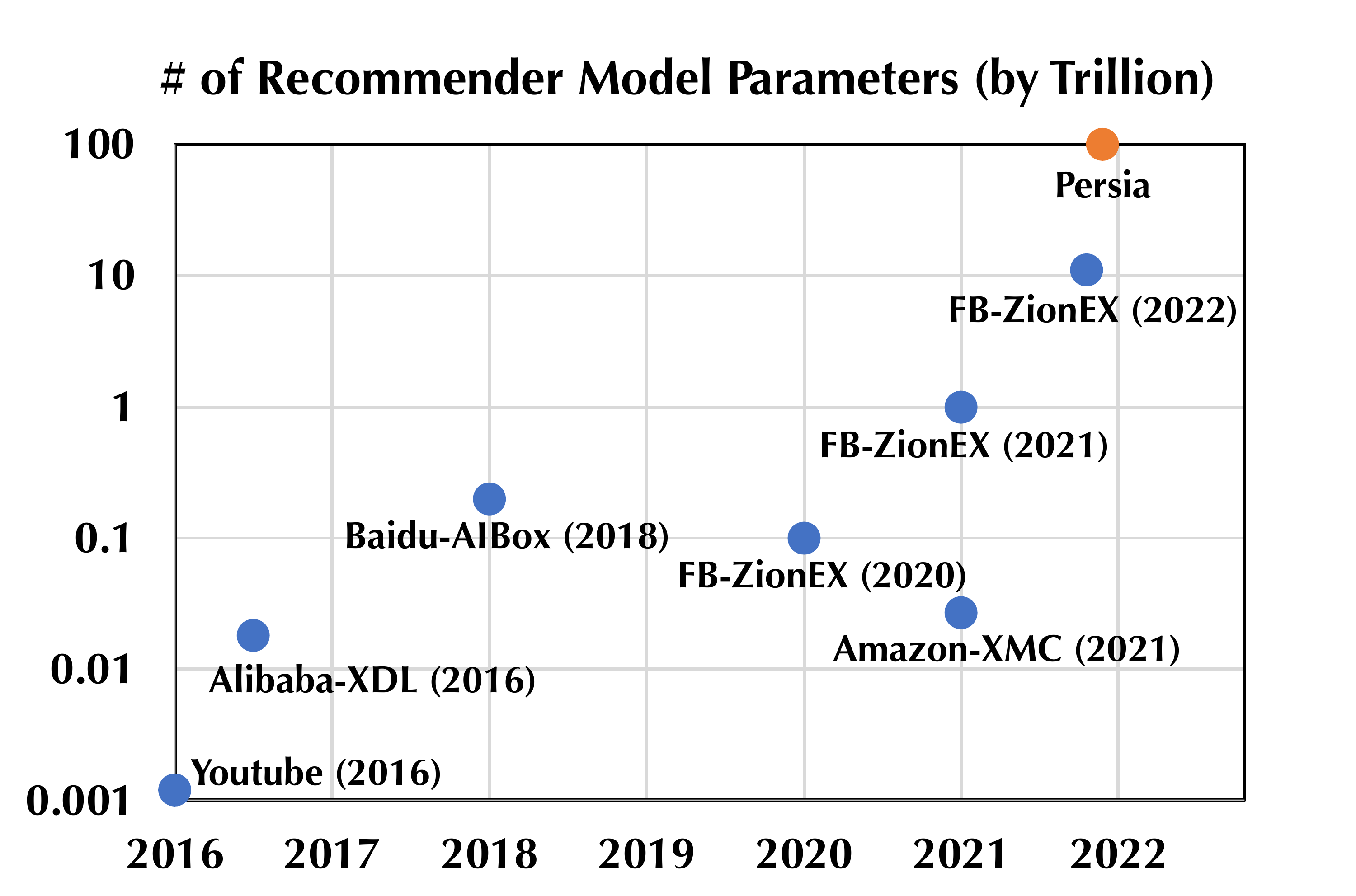}
    \vspace{-1.5em}
    \caption{
    Model sizes of different recommender systems, among which only XDL and AIBox (via PaddlePaddle) are open-source.
    \sys is an open-source training system for deep learning-based recommender systems, which scales up models to the scale of 100 trillion parameters.
    }
    \label{fig:scale}
    \vspace{-1.5em}
\end{figure}

\vspace{0.2em}
\textit{\underline{Racing towards 100 trillion parameters.}}
The continuing advancement of modern recommender models is often driven by the ever increasing model sizes---from Google's 2016 model with 1 billion parameters~\cite{covington2016deep} to Facebook's latest model (2022) with 12 trillion parameters~\cite{mudigere2021high} (See Figure~\ref{fig:scale}).
\textit{Every jump in the model capacity has been bringing in significantly improvement on quality, and the era of 100 trillion parameters is just around the corner.}

Interestingly, the increasing parameter comes mostly from the \textit{embedding layer} which maps each entrance of an ID type feature (such as an user ID~\cite{li2018learning,tang2018personalized} and a session ID~\cite{tan2016improved, twardowski2016modelling, tuan20173d}) into a fixed length low-dimensional embedding vector. Consider the billion scale of entrances for the ID type features in a production recommender system (e.g., \cite{eksombatchai2018pixie,wang2018billion}) and the wide utilization of feature crosses~\cite{cheng2016wide}, the embedding layer usually domains the parameter space, which makes this component extremely \textit{memory-intensive}.
On the other hand, these low-dimensional embedding vectors
are concatenated with diversified Non-ID type features (e.g., image~\cite{wang2017your,wen2018visual}, audio~\cite{van2013deep,wang2014improving}, video~\cite{chen2017attentive,lee2018collaborative}, social network~\cite{hsieh2016immersive,deng2016deep}, etc.) to feed a group of increasingly sophisticated neural networks (e.g., convolution, LSTM, multi-head attention) for prediction(s)~\cite{covington2016deep,zhou2018deep,zhou2019deep,zhang2020model,naumov2020deep,zhao2020autoemb}. Furthermore, in practice, multiple objectives can also be combined and optimized simultaneously for multiple tasks
~\cite{wang2016multi,wang2018explainable, lu2018like, zhao2019recommending,li2020improving}. These mechanisms make the rest \textit{neural network} increasingly \textit{computation-intensive}. 

\vspace{0.2em}




\textit{\underline{Challenges.}} This paper is motivated by the challenges that we were facing when applying existing systems, e.g., XDL from Alibaba~\cite{jiang2019xdl} and PaddlePaddle from Baidu~\cite{paddle}, to the training of models at the 100 trillion parameter scale.
This is a challenging task---as illustrated in Figure~\ref{fig:scale}, most open source systems were only designed for a scale that is at least one order of magnitude smaller; even the largest proprietary systems such as the one reported by Facebook~\cite{mudigere2021high} just two month ago (September 2021) are still $8.3\times$ smaller.
\textit{The goal of this paper is to enable, to our best knowledge, the first open source system that is able to scale in this regime.}


As 100 trillion parameters require at least 200TB to simply store the model (even in \texttt{fp16}), a distributed training system at this scale often consists of hundreds of machines. 
As so, the communication among workers is often a system bottleneck.
While state-of-the-art systems~\cite{jiang2019xdl,raman2019scaling,zhao2019aibox,pan2020dissecting,zhao2020distributed,naumov2020deep,krishna2020accelerating,wilkening2021recssd}
often employ a carefully designed \textit{heterogeneous} architecture (e.g., CPU/GPU, DRAM/SSD) to accommodate the \textit{heterogeneity} in a recommender model (as illustrated in Figure~\ref{fig:problem}), all of them are using a \textit{homogeneous} training algorithm (either synchronous or asynchronous stochastic gradient based algorithms), for the model as a whole.
At the scale of 100 trillion parameters and hundreds of workers, this homogeneity in the training algorithm starts to become a significant problem:

\vspace{0.3em}
\noindent
(1) Synchronous algorithms always use the up-to-date gradient to update the model to ensure the model accuracy. However, the overhead of communications for \textit{synchronous algorithms} starts to
become too expensive to scale out the training procedure, causing inefficiency in running time.

\vspace{0.2em}
\noindent
(2) While \textit{asynchronous algorithm}
have better hardware efficiency, it often leads to a ``significant'' loss in model accuracy at this scale---for production recommender systems (e.g., Baidu's search engine~\cite{zhao2020distributed}). Recall that even 0.1\% drop of accuracy would lead to a noticeable loss in revenue---this is also consistent with our observation at \kuaishou.

Motivated by these challenges, we ask:
\begin{quote}
\em Q1. Can we design an algorithm that can take benefits from both synchronous and asynchronous updates avoiding their disadvantages, to further scale up a recommender system with 100 trillion parameters?
\end{quote}
\begin{quote}
\em Q2. How can we design, optimize, and implement a system to efficiently support such an algorithm?
\end{quote}

\vspace{0.2em}
\textit{\underline{Persia.}} In this paper, we describe \sys, an open source distributed training system 
developed at \kuaishou to support models at the scale of 
100 trillion parameters. 
As we will show in Section~\ref{sec:eval}, this is only possible given its significant speed-up over state-of-the-art systems including both XDL and PaddlePaddle, achieved by the careful \textit{co-design of both the training algorithm and the training system.}
\underline{\em Open Source and Reproducibility:}
We also believe that the ability to train such a large model should be made easily and widely available to everyone instead of just being locked in the hands of a handful of largest companies. We make \sys open source so that everyone who has access to cloud computing service, e.g., Google cloud platform (where we perform \sys's capacity test) can easily setup a distributed training system to reproduce our 100-trillion-parameter model test and to train their own models at this scale\footnote{One can follow our detailed configuration of \sys over Google cloud platform at this link: \url{https://github.com/PersiaML/tutorials/blob/main/src/kubernetes-integration/index.md}).}.

\begin{figure}[t!]
    \centering\includegraphics[width=.475\textwidth]{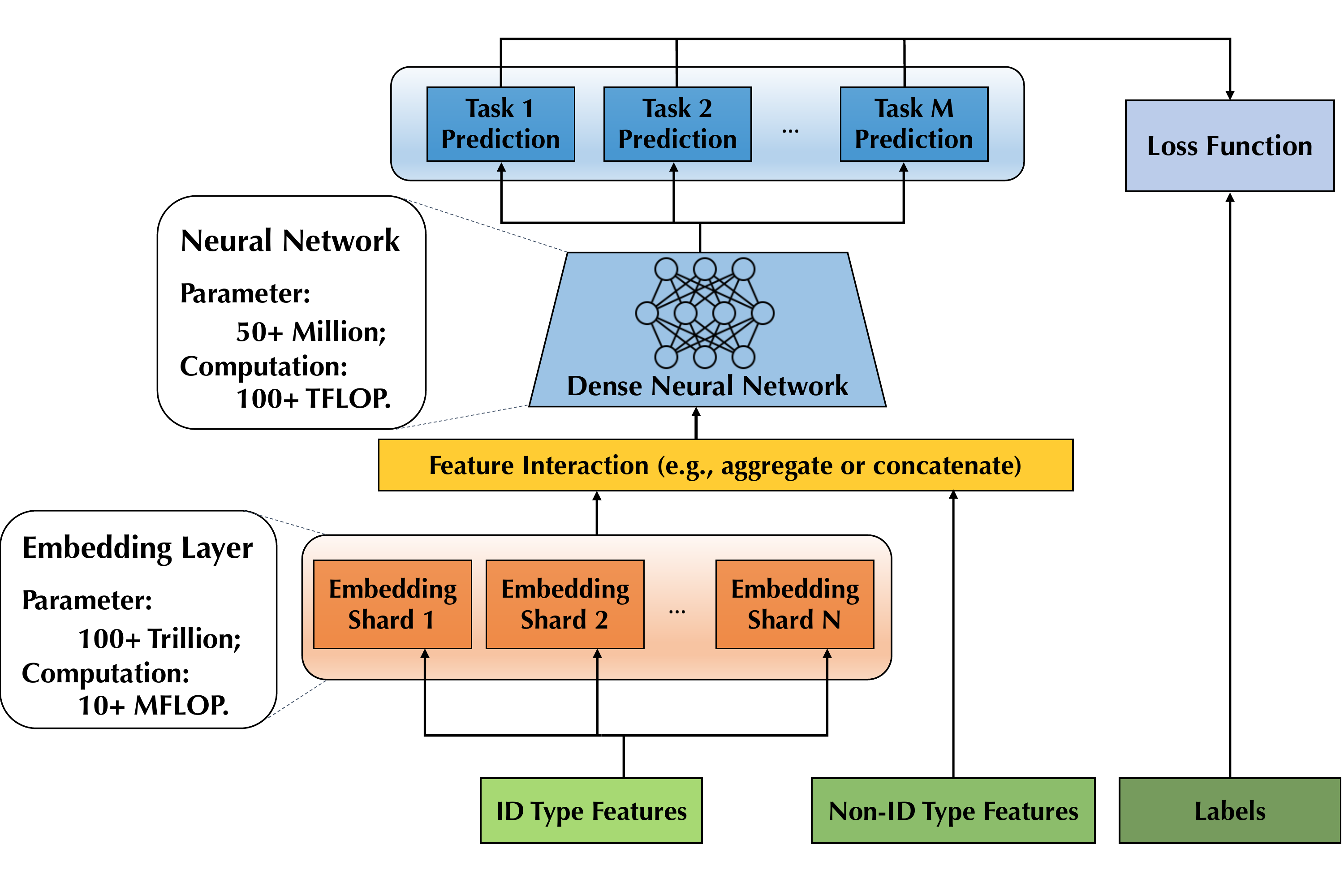}
    \vspace{-2em}
    \caption{An example of a recommender models with $100+$ trillions of parameter in the embedding layer and $50+$ TFLOP computation in the neural network.}
    \label{fig:problem}
    \vspace{-1em}
\end{figure}

\vspace{0.2em}
\textit{\underline{Contributions.}}
\sys is enabled by a set of technical contributions.
The core technical hypothesis of \sys is that, \textit{by using a hybrid and heterogeneous training algorithm, together with a heterogeneous system architecture design, we can further improve the performance of training recommender systems over state-of-the-arts.}

Our \textbf{\underline{first contribution}} is a natural, but novel hybrid training algorithm to tackle the embedding layer and dense neural network modules differently---the embedding layer is trained in an asynchronous fashion to improve the throughput of training samples, while the rest neural network is trained in a synchronous fashion to preserve the statistical efficiency.
We also provide a rigorous theoretical analysis on its convergence behavior and further connect the characteristics of a recommender model to its convergence to justify its effectiveness.

This hybrid algorithm requires us to revisit some decisions with respect to the system architecture and optimizations, to unleash its full potential. 
Our \textbf{\underline{second contribution}} is on the system side.
We design a distributed system to manage the hybrid computation resources (CPUs and GPUs) to optimize the co-existence of asynchronicity and synchronicity in the training algorithm.
We further implement a wide range of system optimizations for memory management and communication optimization---such system optimizations are the key to fully unleash the potential of our hybrid training algorithm. 
We also develop a fault tolerance strategy to handle various potential failures during the training procedure.

\vspace{0.2em}
\textbf{\underline{Last but not least,}} we evaluate \sys using both publicly available benchmark tasks and real-world tasks at \kuaishou. 
We show that \sys scales out effectively and leads to up to $7.12\times$ speedup compared to alternative state-of-the-art approaches~\cite{jiang2019xdl, paddle}.
We further conduct larger-scale scalability experiments up to 100 trillion parameters
on public clouds to ensure public reproducibility.

\vspace{0.2em}
\textit{\underline{Overview.}}
The rest of the paper is organized as follows. We first provide some preliminaries for deep learning recommender systems in Section \ref{sec:pre}. We introduce our hybrid training algorithm in Section \ref{sec:algorithm} and discuss \sys system design and implementation in Section \ref{sec:design}. 
We show the theoretical analysis of the hybrid algorithm in Section \ref{sec:theorem}, present the experimental results in Section \ref{sec:eval}, summarize related work in Section \ref{sec:rel}, and conclude in Section \ref{sec:con}.

\section{Preliminaries}
\label{sec:pre}

This section introduces the basic principle of a typical deep learning based recommendation system (e.g., DLRM \citep{mudigere2021high}) illustrated in Figure~\ref{fig:problem}. We first formalize the distributed recommender training problem; and then give an anatomy of the existing architectures.

\subsection{Problem Formalization}

A typical recommender system takes the training sample in the form:
\begin{align*}
\left[\idInput_\xi,\nidInput_\xi,\mathbf{y}_\xi \right]
\end{align*}
where $\xi$ denotes the index of the sample in the whole dataset. $\idInput_\xi := \{ x_{1; \xi}, x_{2; \xi}, \ldots \}$ is the collection of ID type features in the sample, $\left|\idInput_\xi\right|$ denotes the number of IDs, $\nidInput_\xi$ denotes the Non-ID type features, and $\mathbf{y}_{\xi}$ denotes the label. The ID type feature is the sparse encoding of large-scale categorical information. For example, one may use a group of unique integers to record the microvideos (e.g., noted as $\langle \texttt{VideoIDs}\rangle$) that have been viewed by a user; similar ID type features may include location ($\langle\texttt{LocIDs}\rangle$), relevant topics ($\langle\texttt{TopicIDs}\rangle$), followed video bloggers ($\langle\texttt{BloggerIDs}\rangle$), etc. In our formalization, $\idInput$ is the collection of all ID type features---for the above example, it can be considered as:
\begin{align*}
\idInput \assign \left[ \langle \texttt{VideoIDs}\rangle , \langle\texttt{LocIDs}\rangle, \langle\texttt{TopicIDs}\rangle,  \langle\texttt{BloggerIDs}\rangle, ...\right]
\end{align*}
The Non-ID type feature $\nidInput$ can include various visual or audio features. And the label $\mathbf{y}$ may include one or multiple value(s) corresponding to one or multiple recommendation task(s).

The parameter $\w$ of the recommender system usually has two components:
\begin{align*} 
\w \assign \left[\embWeight, \nnWeight\right] \in \mathbb{R}^{\embDimension + \nnDimension}
\end{align*}
where $\embWeight \in  \mathbb{R}^{\embDimension}$ is the parameter of the embedding layer and $\nnWeight \in  \mathbb{R}^{\nnDimension}$ is the parameter of the rest dense neural network.
We use: $\lookup\left(\idInput\right)$ to denote the concatenation of all
embedding vectors that has correspondence in $\idInput$; $\nn\left(\cdot\right)$ to denote a function parameterized by $\nnWeight$ implemented by a deep neural network that takes the looked up
embeddings and Non-ID features as input and generates the prediction. In the remaining part of this paper, ``NN'' is short for ``neural network''.
The recommender system predicts one or multiple values $\hat{\mathbf{y}}$ by:
\begin{align*} 
\hat{\mathbf{y}}_\xi = \nn\left( \lookup\left(\idInput_\xi\right), \nidInput_\xi \right)
\end{align*}

\textit{It is worth mentioning that while the $\nnWeight$ involved computation can be $10^7\times$ more than the $\embWeight$ involved computation, the size of $\embWeight$ can be $10^7\times$ larger than that of $\nnWeight$, especially when $\embWeight$ contains many cross features.} The imbalance between model scale and model computation intensity is one of the fundamental reasons why the recommender system needs an elaborate training system other than a general purpose deep learning system. 
Formally, the training system essentially solves the following optimization
\begin{align*}
 \min_{\w} \quad \f{\w}
 \assign & \expect{\xi}{\F{\w
  ; \xi}}.\numberthis \label{eq:deciding-dory}
\end{align*}
if we use $\mathcal{L}$ to denote some loss function over the prediction and the true label(s) $\mathbf{y}$, $F$ can be materialized as:
\begin{align*}
F(\w; \xi) \assign \mathcal{L}\left(\nn \left(\lookup\left(\idInput_{\xi}\right), \nidInput_{\xi}\right), \mathbf{y}_{\xi} \right).
\end{align*}
Lastly, we use 
\begin{align*}%
	&\embGradient :=  \nabla_{\embWeight}\F{\embWeight, \nnWeight; \xi}\\
	&\nnGradient :=  \nabla_{\nnWeight}\F{\embWeight, \nnWeight; \xi}
\end{align*}
to denote the gradients of $\embWeight$ and $\nnWeight$ respectively.

\subsection{Anatomy of Existing Architectures}

Existing distributed systems for deep learning based recommender models are usually built on top of the parameter server (PS) framework~\cite{li2014scaling}, where one can add elastic distributed storage to hold the increasingly large amount of parameters of the embedding layer. On the other hand, the computation workload does not scale linearly with the increasing parameter scale of the embedding layer---in fact, with an efficient implementation, a lookup operation over a larger embedding table would introduce almost no additional computations. Thus, a fundamental goal of these systems is to resolve the inconsistency between the computation-intensive neural network and the memory-intensive embedding layer.

One natural solution is to deploy the training task over a hybrid infrastructure as an extension of the original PS framework: the large number of embedding parameters can be sharded across multiple CPU PS nodes for storage and update, whilst the intensive FLOP computation can be assigned to GPU nodes. Thus, the optimization in these systems focuses on avoiding frequent communications of the memory-intensive embedding layer. Towards this end, different designs have been proposed. For example, Alibaba proposes XDL~\cite{jiang2019xdl} that moves the computation of the embedding layer from the GPU worker to the CPU PS node so that only the embedding outputs (instead of all parameters) are transformed to the GPU workers along with the neural network parameters; Baidu designs a hierarchical PS framework~\cite{liu2020distributed} using a hierarchical storage to cache the frequently used parameter close to the GPU workers---only the infrequently used parameters would trigger the slow communication from SSD to GPU.

\textbf{Our Motivation.} 
Existing systems such as XDL and PaddlePaddle update the model in either the pure synchronous fashion or the pure asynchronous fashion. Both generally perform well on training small models. 
However, when training a large-scale model, the synchronous updating suffers from low hardware efficiency while the asynchronous updating usually incurs low accuracy (recall that a tiny loss in accuracy often means a huge loss in revenue in the recommender system). Therefore, it motivates us to design a (sync-async) hybrid algorithm and an efficient system that is able to take benefits from both synchronous and asynchronous updates but avoid their disadvantages.



\begin{figure*}[t!]
    \centering\includegraphics[width=.99\textwidth]{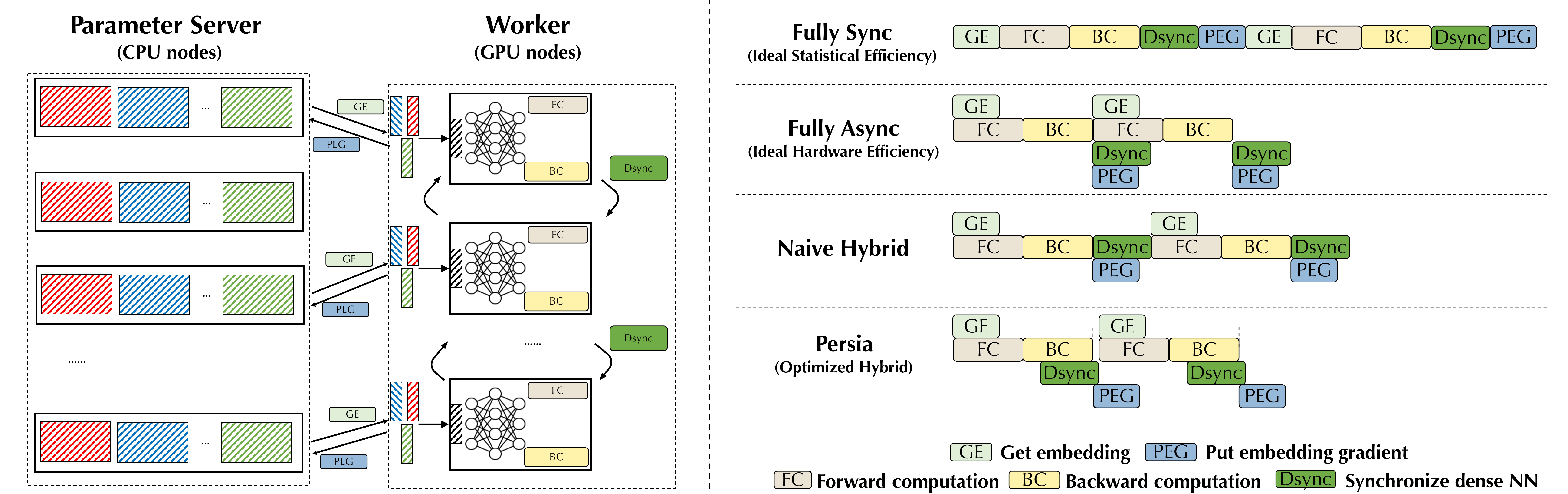}
    \vspace{-0.5em}
    \caption{Left: deep learning based recommender model training workflow over a heterogeneous cluster. Right: Gantt charts to compare fully synchronous, fully asynchronous, raw hybrid and optimized hybrid modes of distributed training of the deep learning recommender model.}
    \label{fig:workflow}
\end{figure*}

\section{Hybrid Training Algorithm}
\label{sec:algorithm}

This section starts with explaining the intuition of the proposed hybrid algorithm, followed by the detailed algorithm description.


\vspace{-0.5em}
\subsection{Why Hybrid Algorithm?}

Training a recommendation mode requires the following essential steps in each iteration as shown in Figure \ref{fig:workflow}-left:


\begin{itemize}[topsep=5pt, leftmargin=*]
\item Preparation of embedding for the training sample(s);
\item Forward propagation of the neural network; 
\item Backward propagation of the neural network;
\item Synchronization of the parameters of the neural network;
\item Update of embedding based on the corresponding gradients.
\end{itemize}

One illustration of the training workflow is shown on in Figure \ref{fig:workflow}-left.
Note that a homogeneous view of the embedding layer and NN would set obstacles for either hardware or statistical efficiency.
For example, consider the fully synchronous mode that we illustrate in Figure \ref{fig:workflow}-right, the above five steps have to happen sequentially---it would be hard to speed up the training procedure based on this mechanism with low hardware efficiency.

\vspace{0.2em}
\textit{\underline{Asynchronous distributed training.}}
It is important to consider appropriate system relaxations for stochastic gradient-based optimizations\footnote{This category of optimization includes SGD, Adam\cite{kingma2014adam}, etc. We slightly abuse the term ``SGD'' to generally refer such optimizations for the rest of the paper.} deployed in a hybrid distributed runtime. There are two key observations from the ML community about asynchronous distributed training of neural networks:
\begin{itemize}[topsep=5pt, leftmargin=*]
\item \textit{Asynchronous updating is efficient with sparse access.}
When individual updating (e.g., SGD iteration) only modifies a small portion of the model's parameters, overwrites are rare and introduce barely no bias into the computation when they do occur~\cite{niu2011hogwild, liancomprehensive}.
\item \textit{Staleness limits the scalability and convergence of asynchronous SGD.} On the other hand, when the updates are heavily overlapped (e.g, dense neural networks), the bias introduced by the discrepancy would limit SGD's scalability and convergence~\cite{lian2015asynchronous, chen2016revisiting}.
\end{itemize}


Without considering the statistical efficiency, the fully asynchronous mode (second row in Figure \ref{fig:workflow}-right) would provide the \textit{optimal hardware efficiency} for distributed recommender model training, where the time of preparing embeddings, synchronizing dense parameters, and updating embedding parameters can be hidden within the computation time for forward and backward propagations of the dense neural network. Unfortunately, as we will illustrate in Section \ref{sec:exp_alg}, the asynchronicity would \textit{hurt the statistical efficiency} and diminish the generalization performance, which is not acceptable for production recommender models ~\cite{zhao2020distributed}.

\vspace{0.5em}
\textbf{Hybrid training algorithm}.  
Based on these observations, one may consider to design different mechanisms for the embedding and dense neural network separately to optimize the training efficiency. 
For this purpose, we introduce a natural, but novel sync-async hybrid algorithm, where the embedding module trains in an asynchronous fashion while the dense neural network is updated synchronously.
Considering the inherited heterogeneity of deep learning based recommender systems, this design would exploit the strength whilst avoid the weaknesses of asynchronous SGD from the algorithmic perspective. Briefly, as we illustrate in the third and forth rows in Figure \ref{fig:workflow}-right, a naive hybrid mode is able to hide the steps of preparing embeddings, and updating embedding parameters within the synchronous training of the dense neural network; further, with some advanced system optimizations (e.g., overlapping computation and communication listed in Section \ref{sec:comm}), the synchronization of the dense parameters would also be able to mostly hidden within the backward propagation of the dense neural network. Thus, \textit{the hybrid algorithm would achieve almost similar hardware efficiency as the fully asynchronous mode without sacrificing the statistical efficiency.}

\vspace{-0.5em}
\subsection{Asynchronous updating for $\embWeight$}
\label{sec:alg_sparse}
To discuss the asynchronous algorithm to update embedding parameter $\embWeight$, given training sample indexed by $\xi$, we introduce the following operations:

\begin{itemize}[topsep=5pt, leftmargin=*]
    \item \texttt{get}$\left(\idInputIndexed{\xi}\right)$: fetch the subset of parameters to generate the embedding of the ID type feature $\idInputIndexed{\xi}$
    ---we use the notation $\embWeightIter{\xi}$ to represent this subset of parameters activated by the sample.

    \item \texttt{put}$\left(\idInputIndexed{\xi}, \embGradientIter {\xi}\right)$: communicate the gradient $\embGradientIter {\xi}$ w.r.t $\idInputIndexed{\xi}$ to the storage of parameter $\embWeight$.

\end{itemize}

Algorithm \ref{alg:sparse} shows the asynchronous updating algorithm for the embedding layer $\embWeight$. Both forward computation task and backward computation task will be executed \textit{without any lock for synchronization}: the forward computation task takes training samples, retrieves the embedding vectors, and sends the corresponding embedding vectors to the proceeding compute units (which handle the NN SGD computation); the backward computation task will receive the gradients of the embedding vectors, and sends it to the parameter storage for parameter updates.

\begin{algorithm}[t!]
\small
 \begin{algorithmic}
 \STATE {\bfseries Context:}{ Forward task for embedding layer.\\}
 \STATE{\bfseries Input:}{ Embedding layer $lookup\left(\cdot\right)$, training data $\mathcal{D}$.\\}
 
 \WHILE{Not converge}
    \STATE{\textit{\color{blue}/* Without any lock: */}\\}
    \STATE{Select a sample from the training set:\: $\idInput_{\xi} \sim \mathcal{D}$;\\}
    \STATE{Get embedding vector(s) w.r.t the sample:\: $\embWeightIter{\xi} \leftarrow$ \texttt{get}$\left(\idInputIndexed{\xi}\right)$;\\}
    \STATE{Send the embedding vector(s) to proceeding units: \: $\embWeightIter{\xi}\downarrow$.\\}
 \ENDWHILE
 \\
 \hrulefill \\
 \STATE{\bfseries Context:}{ Backward task for embedding layer.}
 \STATE {\bfseries Input:}{ Embedding layer $lookup\left(\cdot\right)$, optimizer $\Omega^{\mathrm{emb}}$.}
 \WHILE{Activation's gradients keeps arriving}
    \STATE{\textit{\color{blue}/* Without any lock: */}\\}
    \STATE{Receive gradient of embedding: \: $\embGradientIter{\xi}$;\\}
    \STATE{Send gradients to parameter storage: \: \texttt{put}$\left(\idInputIndexed{\xi},\embGradientIter{\xi}\right)$;\\}
    \STATE{Update embedding parameter: \:$\embWeightIter{t+1} \leftarrow \embWeightIter{t} - \Omega^{\mathrm{emb}}\left(\embGradient \left\{\xi\right\} \right) $.\\}
 \ENDWHILE \\
 \end{algorithmic}
 \caption{Asynchronous updating algorithm for $\embWeight$.}
 \label{alg:sparse}

\end{algorithm}

\begin{algorithm}[t!]
\small
  \begin{algorithmic}
 \STATE{\bfseries Context:}{ A task for dense module (indexed by $k$, $k=1,2,...,K$).}
 \STATE{\bfseries Input:}{ Neural network $\nn\left(\cdot\right)$, embeddings $\embWeightIter{\left(\cdot\right)}$, optimizer $\Omega^{\mathrm{nn}}$.}
 \WHILE{Not converge}
    \STATE{Randomly select $b$ buffered embeddings: \: $\embWeightIter{\xi_1},..., \embWeightIter{\xi_b}\sim \embWeightIter{\left(\cdot\right)}$;\\}
    \STATE{Forward of batch B: \: $F_B = \displaystyle\sum\limits_{i=1}^b \mathcal{L}\left(\nn\left(\embWeightIter{\xi_i}, \nidInput_{\xi_i}\right), \mathbf{y}_{\xi_i}\right)$;\\}
    \STATE{\textit{\color{blue}/* With locks as synchronization barrier: */}\\}
    
    \STATE{Backward of batch B:\: compute gradient $\nnGradientIter{B}$, and $\embGradientIter{\xi_i}$ ;\\}

    \STATE{Sync gradients \textit{with optimization}: \:$ \nnGradientIter{B} \leftarrow \frac{1}{K}\displaystyle\sum\limits_{k=1}^{K} \left[\nnGradientIter{B}\right]_k$}

    \STATE{Update dense parameter:\: $\nnWeightIter{t+1}\leftarrow \nnWeightIter{t} -  \Omega^{\mathrm{nn}} \left(\nnGradientIter{B}\right)$. \\}
    \STATE{Send back activation's gradients:\: $\left[ \embGradientIter{\xi_1}, \embGradientIter{\xi_2} ...\embGradientIter{\xi_b} \right]\uparrow$}
  \ENDWHILE
\end{algorithmic}
 \caption{Synchronous updating algorithm for $\nnWeight$.}
 \label{alg:dense}
\end{algorithm}

\vspace{-0.5em}
\subsection{Synchronous updating for $\nnWeight$}
\label{sec:alg_dense}
Our training algorithm for the rest neural network looks similar to the standard distributed training of deep neural networks---the key difference is that our algorithm concatenates the embedding activations as part of input. The algorithm is illustrated in Algorithm \ref{alg:dense}. Notice that we adopt a mini-batch based SGD algorithm in contrast to the sample-based SGD in the asynchronous counterpart. 


Since the hybrid algorithm is quite different from what has been assumed by existing systems~\cite{li2014scaling,sergeev2018horovod,jiang2019xdl,liu2020distributed,paddle}.
To fully unleash its potential, we have to carefully design the system and optimize its performance---we introduce the system design and implementation of \sys in Section \ref{sec:design}. On the other hand, it is also important to understand the statistical efficiency of the proposed hybrid training algorithm, thus we provide a theoretical analysis about the convergence guarantee of the hybrid algorithm in Section \ref{sec:theory}.

\section{System Design and Implementation}
\label{sec:design}

In this section, we first introduce the design of \sys to support the hybrid algorithm; then we discuss a wide range of implementations to optimize the computation and communication utilization.

\vspace{-0.5em}
\subsection{System Design}

The system design includes two main fundamental aspects: i) the placement of the training workflow over a heterogeneous cluster, and ii) the corresponding training procedure over the hybrid infrastructure. To enlighten the system implementation, we also list the implementation goals of the \sys system here.  

\vspace{0.5em}
\textbf{Workflow placement over heterogeneous cluster.} To support the distributed training of deep learning based recommender model, a \underline{straightforward utilization of the PS paradigm} (provided by general purpose deep learning frameworks, e.g., TensorFlow \cite{abadi2016tensorflow}) would place the storage and update of both embedding and NN parameters in a group of PS nodes (i.e., CPU machines) and the computation of forward and backward propagations in a group of worker nodes (i.e., GPU machines). However, this would be far from efficient and even impossible for deployment. For example, such recommender model would easily exceed the GPU RAM; and the uniform view of embedding and NN modules would introduce a large amount of unnecessary network traffic.

Some \underline{optimized PS architectures} are proposed to optimize the training of deep learning recommender models by rearrange the functionalities in the PS paradigm. For example, XDL \cite{jiang2019xdl} designs the advanced model server to extend the original functionality of PS node to manage the learning (forward and backward propagations) of the embedding module. The hierarchical PS architecture proposed by Baidu \cite{zhao2020distributed} adopts a colocated PS framework with a sophisticated caching schema to reduce the communication overhead.

\underline{New challenges for the layout of 100-trillion-parameter models.} In order to support the recommender models with one or two magnitude larger models, \sys should provide efficient \textit{autoscaling}. Thus, we introduce the following modules, where each module can be dynamically scaled for different model scales and desired training throughput:

\begin{itemize}[topsep=5pt, leftmargin=*]
    \item A \textbf{data loader} that fetches training data from distributed storages such as Hadoop, Kafka, etc; 
    \item An embedding parameter server (\textbf{embedding PS} for short) that manages the storage and update of the parameters in the embedding layer $\embWeight$;
    \item A group of \textbf{embedding workers} that runs Algorithm \ref{alg:sparse} for getting the embedding parameters from the embedding PS; aggregating embedding vectors (potentially) and putting embedding gradients back to embedding PS;
    \item A group of \textbf{NN workers} that runs the forward-/backward- propagation of the neural network $\nn\left(\cdot\right)$.
\end{itemize}

Considering the heterogeneity between the embedding and NN modules, \sys adopts different communication paradigms for training process: i) the \textit{PS paradigm} between embedding PS and embedding workers (running on CPU nodes) to manage the training of the embedding layer, while ii) the \textit{AllReduce paradigm} among NN workers (running on GPU nodes) for the NN.

\begin{figure}[t!]
    \centering\includegraphics[width=.475\textwidth]{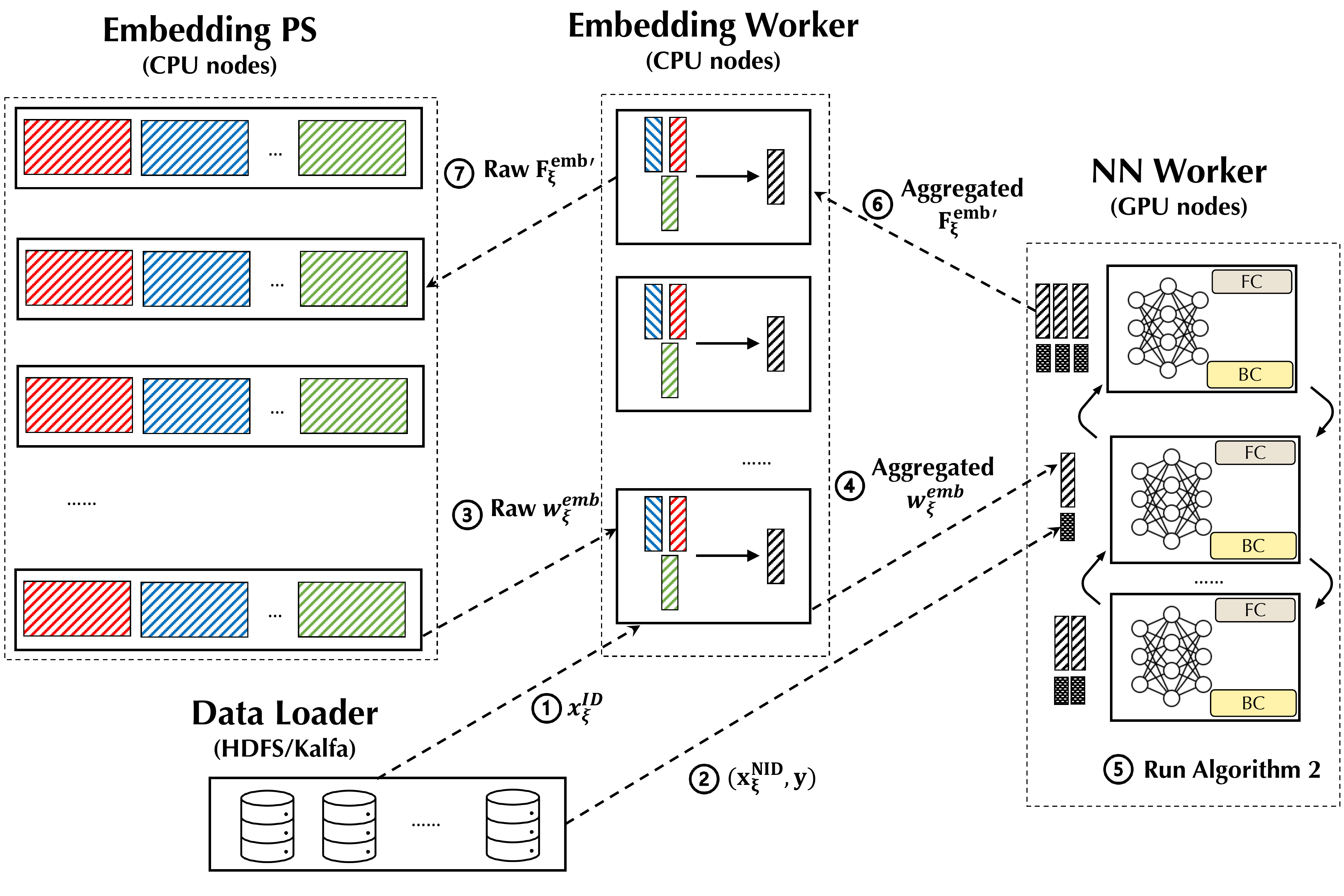}
    \vspace{-1em}
    \caption{The architecture of \sys. \sys includes a data loader module, an embedding PS module, a group of embedding workers over CPU nodes, and a group of NN workers over GPU instances.}
    \label{fig:framework}
    \vspace{-1em}
\end{figure}

\vspace{0.5em}
\noindent\textbf{Distributed training procedure.}
Logically, the training procedure is conducted by \sys in a \textit{data dispatching} based paradigm as below (see Figure \ref{fig:framework}): 
\begin{enumerate}[topsep=5pt, leftmargin=*]
    \item The data loader will dispatch the ID type feature $\idInputIndexed{\left(\cdot\right)}$ to an embedding worker---the embedding worker will generate a unique sample ID\footnote{Note that the unique ID $\xi$ will be used to locate the embedding worker that generates this ID---this could simply be implemented by using the first byte to encode the rank of this embedding worker.} $\xi$ for this sample, buffer this sample ID with the ID type feature $\idInputIndexed{\xi}$ locally, and return this ID $\xi$ back the data loader; the data loader will associate this sample's Non-ID type features and labels with this unique ID.
    
    \item Next, the data loader will dispatch the Non-ID type feature and label(s) $\left(\nidInputIndexed{\xi},\mathbf{y}_\xi\right)$ to a NN worker.
    
    \item Once a NN worker receives this incomplete training sample, it will issue a request to pull the ID type features' ($\idInput_\xi$) embedding $\embWeightIter{\xi}$ from some embedding worker according to the sample ID $\xi$---this would trigger the forward propagation in Algorithm \ref{alg:sparse}, where the embedding worker will use the buffered ID type feature $\idInputIndexed{\xi}$ to \texttt{get} the corresponding $\embWeightIter{\xi}$ from the embedding PS.

    \item Then the embedding worker performs some potential aggregation of original embedding vectors.
    When this computation finishes, the aggregated embedding vector $\embWeightIter{\xi}$ will be transmitted to the NN worker that issues the pull request.
    \item Once the NN worker gets a group of complete inputs for the dense module, it will create a mini-batch and conduct the training computation of the NN according to Algorithm \ref{alg:dense}. Note that the parameter of the NN always locates in the device RAM of the NN worker, where the NN workers synchronize the gradients by the AllReduce Paradigm.
    \item When the iteration of Algorithm \ref{alg:dense} is finished, the NN worker will send the gradients of the embedding ($\embGradientIter{\xi}$) back to the embedding worker (also along with the sample ID $\xi$).
    \item The embedding worker will query the buffered ID type feature $\idInputIndexed{\xi}$ according to the sample ID $\xi$; compute gradients $\embGradientIter{\xi}$ of the embedding parameters and send the gradients to the embedding PS, so that the embedding PS can finally compute the updates according the embedding parameter's gradients by its SGD optimizer and update the embedding parameters.
\end{enumerate}

\textbf{System implementation goals.}
To efficiently support the hybrid training algorithm, we listed the following system design goals for \sys:

\vspace{0.5em}
\noindent\underline{\textit{Fill the asynchronicity and synchronicity gap.}}
One central functionality of \sys system is to handle the heterogeneity inherited from the hybrid training algorithm. This would request \sys to seamlessly connect the forward- and backward- propagation during the training phase---the coordination of embeddings and gradients transmitted in a large-scale cluster is a unique challenge for \sys.

\vspace{0.5em}
\noindent\underline{\textit{Utilize the heterogeneous clusters efficiently.}}
To achieve the actual performance gain, \sys needs to include different mechanisms to fully utilize the computation resources (e.g., CPUs, GPUs) given diversified link connections. This demands an \textit{efficient memory management} and a group of \textit{optimized communication mechanisms}.  

\vspace{0.5em}
\noindent\underline{\textit{Provide effective fault tolerance.}}
With the large number of machines that training requires, effective fault tolerance is necessary. Furthermore, the hybrid algorithm would request more complex mechanisms to manage the heterogeneous computation and communication, which poses additional challenges for fault tolerance.

\vspace{-0.5em}
\subsection{System Optimizations and Design Decisions}

We enumerate the implementation details to fulfill the design goals.

\subsubsection{Fill the Async/Sync Gap}
\label{sec:gap}

To fill the gap between synchronous and asynchronous updates, both embedding and NN workers implement some buffering mechanisms.

\vspace{0.5em}
\textbf{NN worker buffer mechanism.}
Since we adopt a \textit{GPU-pull} based schema for the hybrid training procedure between the NN worker and the embedding worker,
each NN worker will locally maintain an \textit{input sample hash-map} keyed on the sample ID $\xi$ and valued on tuples of Non-ID type feature $\nidInputIndexed{\xi}$ and label $\mathbf{y}_{\xi}$.
In the forward propagation, once a NN worker receives the Non-ID type feature and label, it will first insert the key-value pair $\left(\mathrm{key}:\xi, \: \mathrm{value}:\left(\nidInputIndexed{\xi},\mathbf{y}_{\xi}\right) \right)$ to the input sample hash-map; 
and then send the request of the embedding vector to the embedding worker. 
Later, when the embedding vector $\embWeightIter{\xi}$ arrives in the NN worker from an embedding worker, 
the NN worker will pop the key-value pair $\left(\mathrm{key}:\xi, \: \mathrm{value}:\left(\nidInputIndexed{\xi},\mathbf{y}_{\xi}\right) \right)$ from the input sample hash-map, consume $\embWeightIter{\xi}$, $\nidInputIndexed{\xi}$, $\mathbf{y}_{\xi}$ in the mini-batch for the SGD computation within the GPU. 
In the backward propagation, once the computation is done, the NN worker will use the sample ID $\xi$ to locate the embedding worker, and then send the gradient of the embedding $\embGradientIter{\xi}$ to this embedding worker.

\vspace{0.5em}
\textbf{Embedding worker buffering mechanism.}
To support the pull request from NN worker, an embedding worker also needs to locally maintain a \textit{ID type feature hash-map} that is keyed on sample ID $\xi$ and valued on ID type feature $\idInputIndexed{\xi}$.
In the forward propagation, when the embedding worker receives a new ID type feature, it will generate a unique ID for this sample, store the key-value pair $\left(\mathrm{key}:\xi,\: \mathrm{value}: \idInputIndexed{\xi} \right)$ in the ID type feature hash-map. Then the embedding worker will \texttt{get} the corresponding embedding parameter $\embWeightIter{\xi}$ from embedding PS and return it to the NN worker.
In the backward propagation, once the embedding worker receives the gradient of the embedding vector $\embGradientIter{\xi}$ from a NN worker,
the embedding worker will find the corresponding ID type feature $\idInputIndexed{\xi}$. Note that this is implemented as a local search in the ID type feature hash-map---the sample ID $\xi$ serves as the key to retrieve $\idInputIndexed{\xi}$. Later, the embedding worker can \texttt{put} the gradient of the embedding parameters back to the embedding PS.

\subsubsection{Persia Memory Management}
\label{sec:memory}

Memory management is an important component for \sys to efficiently utilize the hybrid infrastructure, especially for the embedding PS which is responsible for maintaining trillions of embedding parameters. Comprehensively, the embedding PS works like a standard PS, which is not significantly different from a distributed key-value store. As we illustrated in Figure \ref{fig:memory}, when retrieving the embedding parameters, an embedding worker first runs an identical global hashing function to locate the embedding PS node that stores the parameters; once the request arrives in the PS node, the parameter can be acquired in an LRU cache as we explain below.

\begin{figure}[t!]
    \centering\includegraphics[width=.495\textwidth]{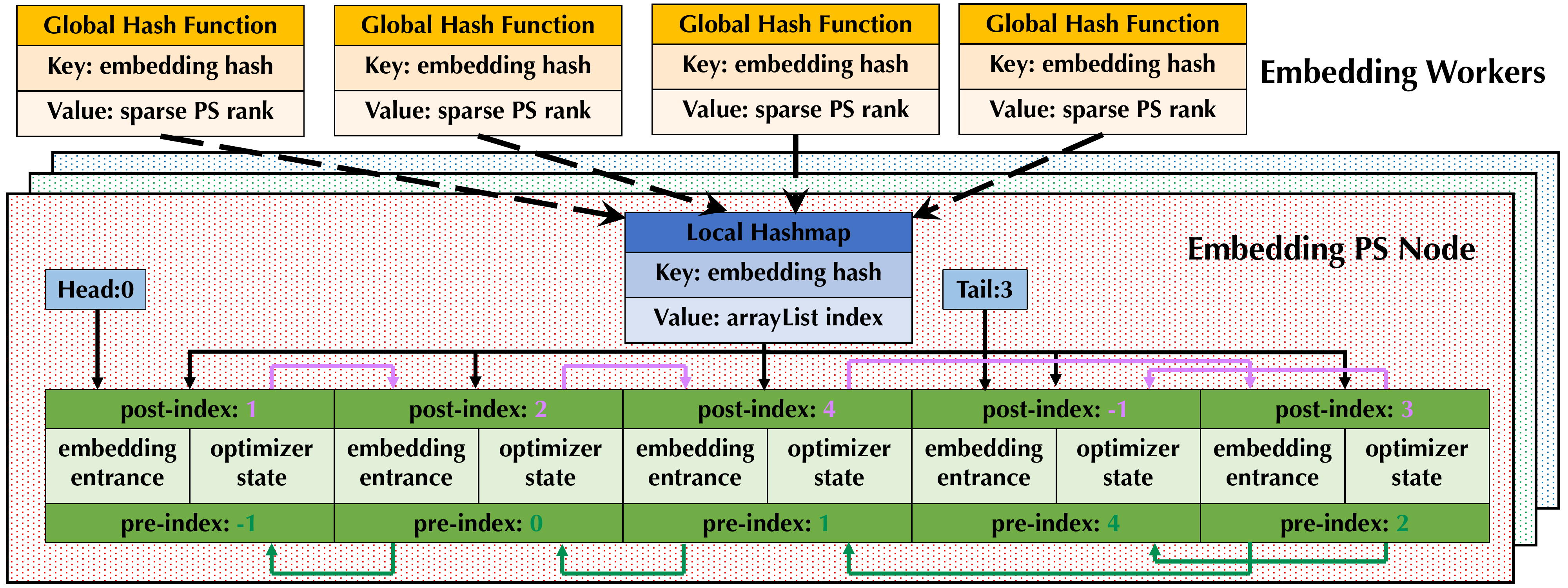}
    \vspace{-1.5em}
    \caption{Memory management of \sys's embedding PS node. The design is based on a LRU cache implemented by hash-map and array-list.}
    \label{fig:memory}
    \vspace{-1.5em}
\end{figure}

\vspace{0.5em}
\textbf{LRU cache implementation.} \sys leverages LRU cache\footnote{Standard LRU implementation is based on hash-map and doubly linked list, e.g., \cite{lru}.} to maintain the embedding parameters in RAM. We use array-list and hash-map to implement the LRU. Instead of a doubly linked list where the pointer stores a memory address, we adopt an array-list design where the pointer stores the index of the pre- or post-  entrance in the array; similarly, the hash-map's value also stores the corresponding embedding parameter's index in the array instead of the memory address. Besides the pre- and post- indices, each item in the array also includes two fields: the embedding vector and the optimizer states corresponding to this embedding vector.
There are two advantages by utilizing an array-based linked list: i) this mechanism avoids frequent allocation and deallocation of fragmented memory blocks---this cost is not negligible when each linked list may contain billions of entrances; ii) since pointer (that stores the memory address) does not exist in the data structure, serialization and deserialization become a straightforward memory copy, which is far more convenient and efficient---this is helpful for periodic saving and loading checkpoints for fault tolerance and model sharing. Note that to fully utilize the CPU cores and RAM in the embedding PS node, we utilize multiple threads in the LRU implementation. Each thread manages a subset of the local hash-map and the corresponding array-list; when there is a request of \texttt{get} or \texttt{put}, the corresponding thread will lock its hash-map and array-list until the execution is completed.




\subsubsection{Communication Optimization}
\label{sec:comm}

To fully utilize the computation power in the heterogeneous cluster, \sys implements
a range of system optimizations to decrease the communication overhead.

\vspace{0.5em}
 \textbf{Optimized communication among NN workers.}
The optimization of AllReduce communication paradigm among NN workers in \sys is the key for \textit{hiding communication overhead within the backward computation} of the neural network. This functionality is implemented based on \bagua~\cite{gan2021bagua}, an open-source general-purposed distributed learning system optimized for data parallelism, also released by \kuaishou.   
Currently, \sys utilizes \bagua's centralized synchronous full-precision communication primitive (equivalent to AllReduce) by default, in an attempt to preserve the accuracy. \sys leverages \bagua for the synchronization among NN workers because additional system optimization enabled by \bagua can be directly adopted, including 
\textit{tensor bucketing, memory flattening}, \textit{hierarchical communications}, etc.
Notice that there are other communication primitives provided by \bagua which could potentially further improve the training throughput; however, it is unclear if these communication primitives would hurt the statistical efficiency, we leave this exploration as an interesting future work.

\vspace{0.5em}
\textbf{Optimized remote procedure call.} The point-to-point communication: i) between NN workers and embedding workers, and ii) between embedding workers and embedding PS is implemented by remote procedure call (RPC). Unlike the traditional usage of RPC, where the communication is mainly responsible for transmitting small objects with complex serialization and deserialization mechanism, \sys demands a RPC implementation that is efficient for communicating tensors stored in a large continuous memory space. As so, \sys abandons the protocol buffer based implementation (adopted by gRPC~\cite{grpc} and bRPC~\cite{brpc} for other learning systems), which introduces significant overhead for communicating tensors; instead, \sys adopts a simple but efficient zero-copy serialization and deserialization mechanism targeting for tensors which directly uses memory layout to serialize and deserialize tensors. Further, as we mentioned in Section \ref{sec:memory}, since tensors on host RAM are allocated in large pages, the TLB lookup time is also reduced significantly, which accelerates the copying procedure from host RAM to the RAM of the network adapter.

\vspace{0.5em}
\textbf{Workload balance of embedding PS.}
\label{sec:embPSwork}
The embedding parameter storage is implemented as a sharded PS to support the query and update of the parameter in the embedding layer. 
The key challenge in embedding PS implementation is about the workload balance of query and update about the embedding parameter. Initially, we adopt a straightforward design by distributing the embedding parameter according to the feature groups. A sub-group of CPU instances are allocated to manage a partition of semantic independent embeddings.
We find that in practice, this would easily lead to congestion in the access of some feature groups during training---
the access of training data can irregularly lean towards a particular embedding group during the online learning procedure in industrial-scale applications. We solve this issue by adopting an alternative partition of the embedding parameter: the embeddings inside a feature group are first \textit{uniformly shuffled} and then \textit{evenly} distributed across embedding PS nodes. We observe that this design effectively diminishes the congestion of the embedding parameter access and keeps a balanced workload for the embedding PS.

\vspace{0.5em}
\textbf{Communication compression.}
To reduce the network traffic, we adopt both lossless and lossy compression mechanisms for the communication request between embedding and NN workers. The network bandwidth connecting GPU instances is limited---besides the AllReduce operation that leverages this bandwidth, another noticeable utility is the communication of embedding vectors in the forward propagation (\textcircled{\small{4}} in Figure \ref{fig:framework}) and its gradient in the backward propagation (\textcircled{\small{6}} in Figure \ref{fig:framework}). The communication between GPUs enabled by advanced connections like GPUDirect RDMA is $10\times$ faster than that between GPU and CPU nodes through PCI-e \cite{gpudirect}, which could leave the communication of embedding activation and its gradients as a bottleneck.
Note that although plenty of lossy compression schema has been proposed for distributed learning, one should be cautious of applying them in distributed recommender model training---for a commercial recommender system, as mentioned before, even a drop of $0.1\%$ accuracy is not affordable~\cite{zhao2020distributed}.
To reduce such network traffic, 
we apply a lossless compression mechanism for the index component and a discreet lossy compression mechanism for the value component.

\vspace{0.5em}
\noindent\underline{\textit{Lossless compression.}} For the index component of the embedding, instead of representing a batch of samples as a list of vectors, where each vector containing all IDs (represented by \texttt{int64}) of a sample, we represent a batch as a hash-map, where the key is unique IDs in the whole batch, and the value corresponding to each unique ID is the indices of the samples in the batch containing this ID. Since the batch size is relatively small ($\leq65535$), the indices can be represented using \texttt{uint16} instead of \texttt{int64} without losing any information.

\vspace{0.5em}
\noindent\underline{\textit{Lossy compression.}} For the value component, We adopt a discreet \texttt{fp32} to \texttt{fp16} compression. Notice that a uniform mapping from \texttt{fp32} to \texttt{fp16} would harm the statistic efficiency significantly, so we defines an nonuniform mapping method:  
suppose $\lVert \cdot \rVert_\infty$ represents the $L_{\infty}$ norm of a vector, $\kappa$ represents a relatively large constant scalar. 
In the compression side, each \texttt{fp32} vector block $\mathbf{v}$ is first scaled by $\frac{\kappa}{\lVert \mathbf{v} \rVert_\infty}$ and then converted to a \texttt{fp16} block vector. In the decompression side, the compressed block vector $\tilde{\mathbf{v}}$ communicated as \texttt{fp16} is first converted back to a \texttt{fp32} vector and then divided by $\frac{\kappa}{\lVert \mathbf{v} \rVert_\infty}$. 

\subsubsection{Fault Tolerance}
\label{sec:fault}

Failure can happen frequently considering the large number of nodes that participate in the hybrid training.
Towards this end, \sys implements a group of fault tolerance mechanisms to handle different failures in the cluster.
Two observations are interesting:

\begin{itemize}[topsep=5pt, leftmargin=*]
\item The infrequent loss of parameter update of the embedding layer is usually negligible for convergence, while the responsive time of the embedding parameter query would be important for the end-to-end training time; 
\item Any drop of the model synchronization by the NN worker that runs the dense synchronous training algorithm is vital for convergence.
Based on these observations, \sys implements the following mechanisms for different components to handle different system failures.
\end{itemize}

Since the \underline{data loader} should be able to run on top of any other popular distributed storage systems (e.g., HDFS), \sys relies on their own recovery schema once an instance failure happens. Notice that \sys mainly considers the online training setting, where no shuffling schema is required by \sys's data loader.

The \underline{embedding PS} should be responsive during the hybrid training execution. For this purpose, the embedding PS node will put in-memory LRU cache (introduced in Section \ref{sec:memory}) in a shared memory space---by this fashion, once a process-level failure happens, the process can automatically restart and attach to the consistent shared-memory space without influencing any other instances of the embedding PS. Additionally, embedding PS nodes will periodically save the in-memory copy of the embedding parameter shard, with the advance of our LRU implementation, check-pointing is very efficient (also see Section \ref{sec:memory}).

The \underline{embedding worker} has no fault recovery schema---once a failure happens, the local buffer of the ID type feature hash-map will be simply abandoned without any recovery attempts.

By contrast, the dense module cannot afford any drop of model synchronization. As so, the \underline{NN worker} would also periodically save the synchronized model as the checkpoint with the same frequency as the embedding PS. Once a failure of GPU instances happens, all the GPU instances will abandon their local copy of the model, load the latest checkpoint, and continue the execution of the dense synchronous training algorithm.

\section{Theoretical Analysis} \label{sec:theory}
\label{sec:theorem}
In this section, we show the convergence rate of the proposed hybrid algorithm adopted by \sys. In short, the proposed hybrid algorithm converges and admits a similar convergence rate (or total complexity equivalently) to the standard synchronous algorithm. That is to say that the hybrid algorithm takes a significant benefit from asynchronicity in the system efficiency but sacrifices very little for the convergence rate.  

To show the detailed convergence rate, we first provide the essential updating rule of the hybrid algorithm in \sys: 
\begin{equation}
\begin{aligned}
\embWeightIter{t+1} = &\embWeightIter{t} - \gamma \embGradientIter{t} \\ 
\nnWeightIter{t+1} = & \nnWeightIter{t} - \gamma \nnGradientIter{t}, 
\end{aligned}%
\label{eq:59f6c90b}
\end{equation}
where $\embGradientIter{t}$ and $\nnGradientIter{t}$ are notations for gradients of $\embWeightIter{t}$ and $\nnWeightIter{t}$, short for respectively:
\begin{align*}%
	\embGradientIter{t} := & \nabla_{\embWeight}\F{\embWeightIter{\mathD (t)}, \nnWeightIter{t}; \xi_{t}},\\ 
	\nnGradientIter{t} := & \nabla_{\nnWeight}\F{\embWeightIter{\mathD (t)}, \nnWeightIter{t}; \xi_{t}},%
\end{align*}
where $D(t) \leqslant t$ denotes some iterate earlier than the current iterate $t$. This is because of the asynchronous update for the embedding layer. Since the dense neural network adopts the synchronous update, it always uses the up to date value of $\nnWeightIter{t}$, while the used value of the embedding layer may be from some early iterate. This is the cost of using asynchronous updating, but we will see that the cost is very minor comparing to our gain in system efficiency.

To ensure the convergence rate, we still need to make some commonly used assumptions as follows:
\begin{assumption}%
\label{assumption:main} We make the following commonly used assumptions for analyzing the stochastic algorithm:%

\begin{itemize}[topsep=5pt, leftmargin=*] %

\item \textbf{Existence of global minimum.} Assume \[f^{*}\assign \min_{\w}f\left( \w \right)\] exists.

\item \textbf{Lipschitzian gradient.} The stochastic gradient function
$F'\left(\cdot; \xi\right)$ is differentiable, and $L$-Lipschitzian for all $\xi$:
\begin{align}
& \left \| F' \left(\w; \xi\right) -  F' \left(\bar{\w}; \xi\right) \right\| \leqslant L \left\|\w - \bar{\w} \right\|, \nonumber \\ 
& \quad \forall \xi \in \Xi, \forall \w, \bar{\w} \in \mathbb{R}^{\embDimension+\nnDimension}.\label{eq:able-snipe}
\end{align}

\item {\textbf{Bounded variance.}} The variance of the stochastic gradient is bounded: there exists a constant $\sigma\geqslant 0$ such that
\begin{align}%
\mathbb{E} \left[ \left\| F'\left(\w; \xi \right)-f' \left(\w\right) \right\|^2 \right] \leqslant \sigma^2 \quad \forall \w. \label{eq:credible-bird}
\end{align}

\item {\textbf{Bounded staleness.}} All sparse model update delays are bounded: there exists a constant $\tau \geqslant 0 $ such that%
\begin{align}%
t - D(t) \leqslant & \tau, \quad \forall t. \label{eq:c21fc02e}
\end{align}

In practice $\tau$ is the number of samples whose embeddings are retrieved but their gradients are not yet updated into the model parameters. In \sys this value is less than 5 for most cases.

\end{itemize}
\end{assumption}

\begin{remark}
The assumptions of {\bf existence}, {\bf bounded variance}, and {\bf bounded staleness} are commonly used ones. The {\bf bounded staleness} assumption is due to the asynchronous update for $\embWeight$. In practice, $t-D(t) \leqslant \tau$ means that any element in $\embWeight$ is updated up to $\tau$ times between reading it and writing it within training the same sample. In \sys this value of $\tau$ is less than 5 for most scenarios. The {\bf frequency} assumption reflects the intrinsic property of the recommendation system. For each sample $\xi$, if every ID type feature $x$ has the same probability of being in $\idInputIndexed{\xi}$, then $\alpha = \frac{1}{\embDimension}$. On the other hand, if an ID type feature $x$ is contained by every sample $\xi$'s ID type features $\idInputIndexed{\xi}$, then $\alpha = 1$.
\end{remark}

Then we obtain the following convergence result: 
\begin{theorem}%
	\label{thm:9853187e}
Denote $\alpha$ to be the constant such that for each ID the probability of a sample containing it is smaller than $\alpha$.
  Under Assumption \ref{assumption:main}, with learning rate $\gamma$ in \eqref{eq:59f6c90b} as

\begin{align*}
\gamma = \frac{1}{L + \sqrt{TL} \sigma + 4\tau L \alpha}
\end{align*}
\sys admits the following convergence rate:%

\begin{align}%
    \frac{\sum_{t = 0}^{T - 1} \mathbb{E} \left[ \left\| f' (\w_t) \right\|^2\right]}{T} \lesssim & \underbrace{\frac{\sigma}{\sqrt{T}} + \frac{1}{T} +}_{\text{rate of vanilla SGD}} \frac{\alpha\tau}{T}
             \label{eq:2d1ebf2a}
	\end{align}%
	where $\lesssim$ means
	``smaller than or equal to up to a constant factor'', and $L, G$ are all
	treated as constants for simplicity.
	For example $a_{t} \lesssim b_{t}$ means that there exists a constant $C>0$ such that $a_{t}\le C
		b_{t}$ for all $t$.
\end{theorem}



The complete proof is provided in the preprint version of this paper~\cite{lian2021persia}.
Note that the first two terms in
\eqref{eq:2d1ebf2a} are exactly the convergence rate of vanilla SGD \citep{liu2020distributed}.
The third (additional) term is caused by the staleness of the asynchronous update in the embedding layer of the model. The staleness upper bound $\tau$ is typically proportional to the number of workers. The ID frequency upper bound $\alpha$ is a value smaller than $1$. Notice that $\alpha=1$ leads a convergence rate exactly matching the asynchronous SGD \citep{liu2020distributed}. It indicates that \sys (the hybrid algorithm) guarantees to be no worse than the asynchronous algorithm in convergence rate. When $\alpha \ll 1$, which is mostly always true in real recommendation systems, the hybrid algorithm---\sys---admits a very similar convergence rate to the synchronous algorithm, since the third term is dominated by the second term. Therefore, this theorem suggests that the hybrid algorithm follows the same convergence efficiency in terms of iterations.


\begin{figure*}[t!]
\centering
    \begin{subfigure}
        \centering\includegraphics[width=13.35em]{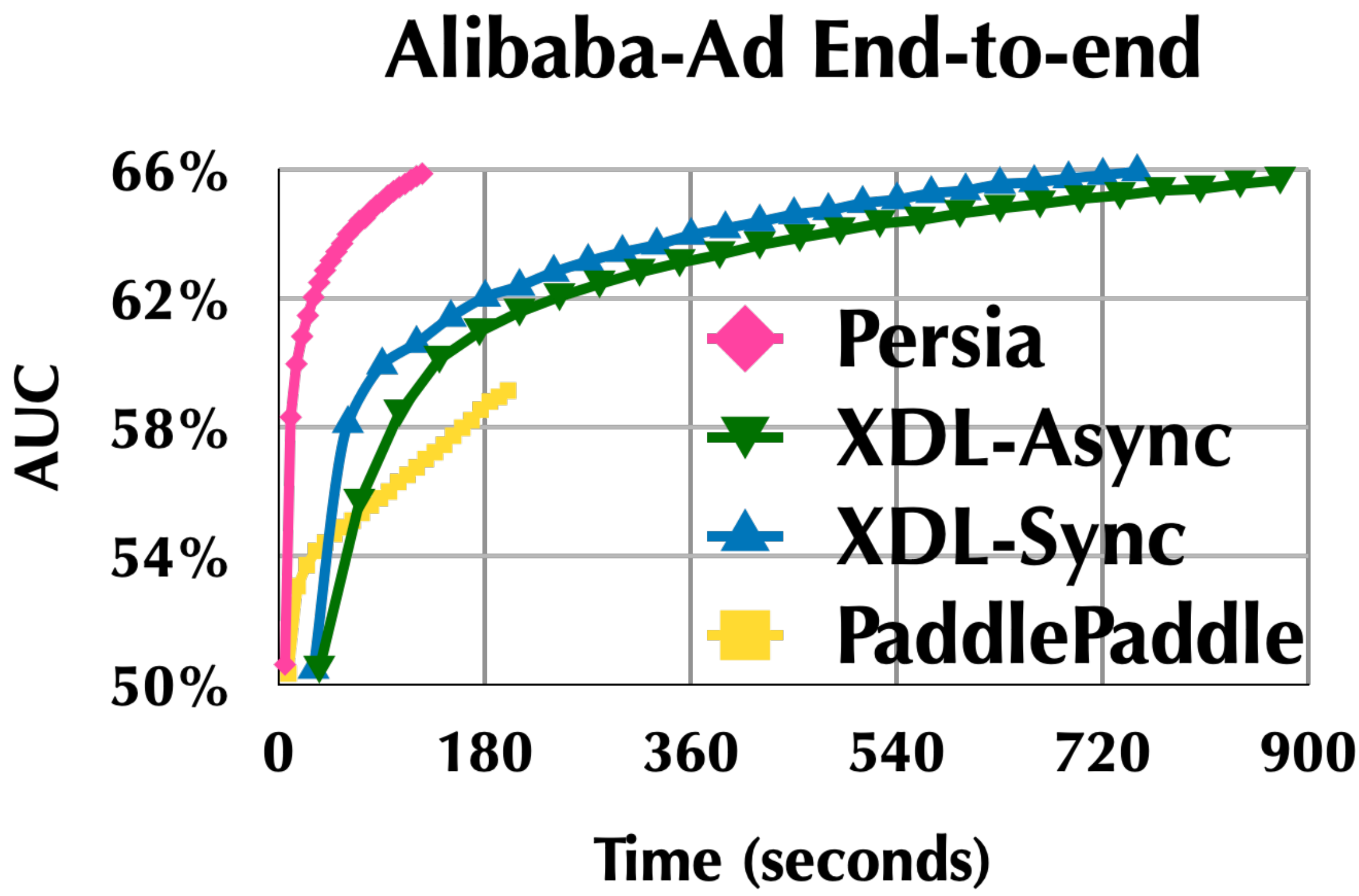}
    \end{subfigure}
    \begin{subfigure}
        \centering\includegraphics[width=13.35em]{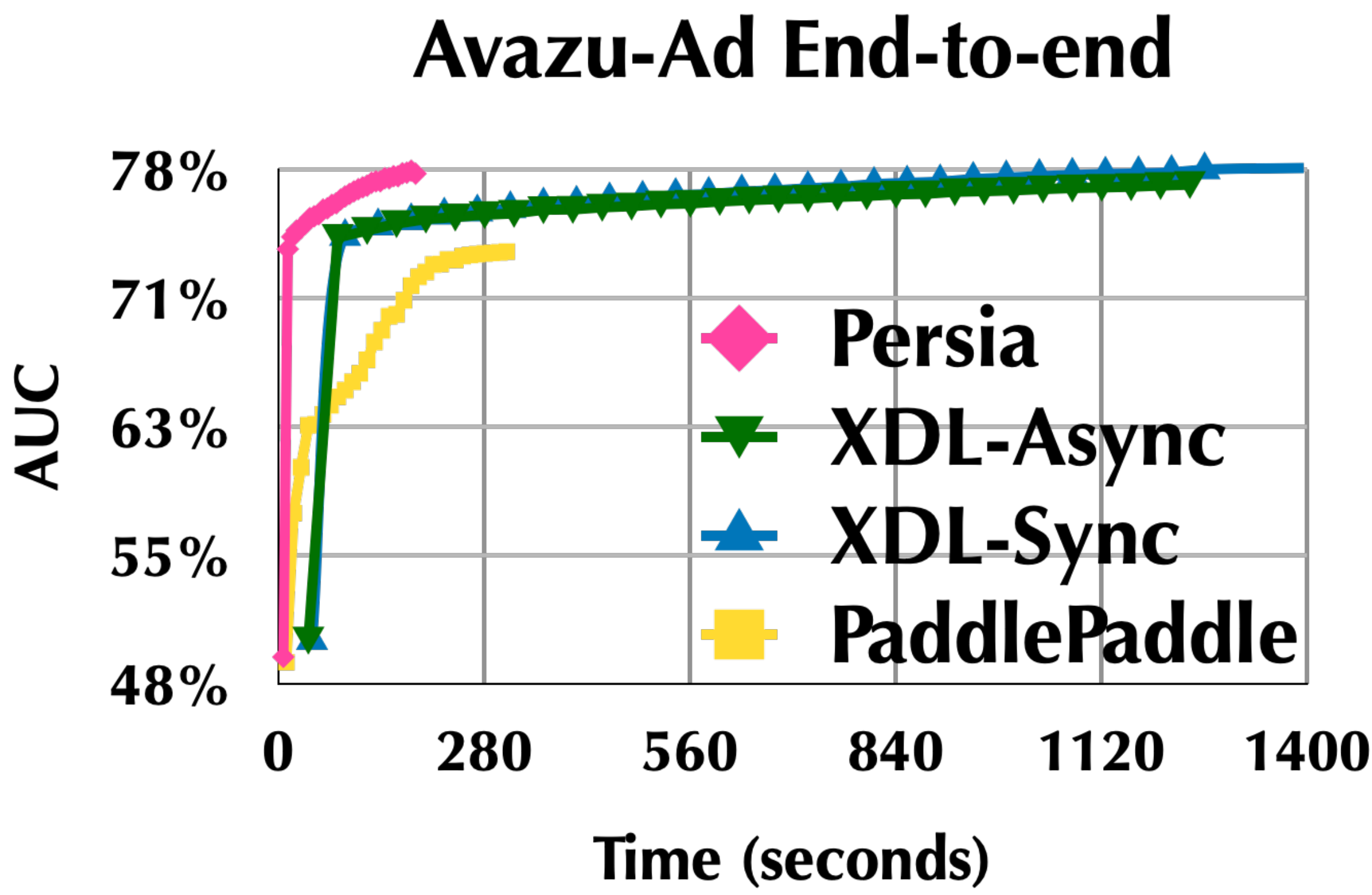}
    \end{subfigure}
    \begin{subfigure}
        \centering\includegraphics[width=13.35em]{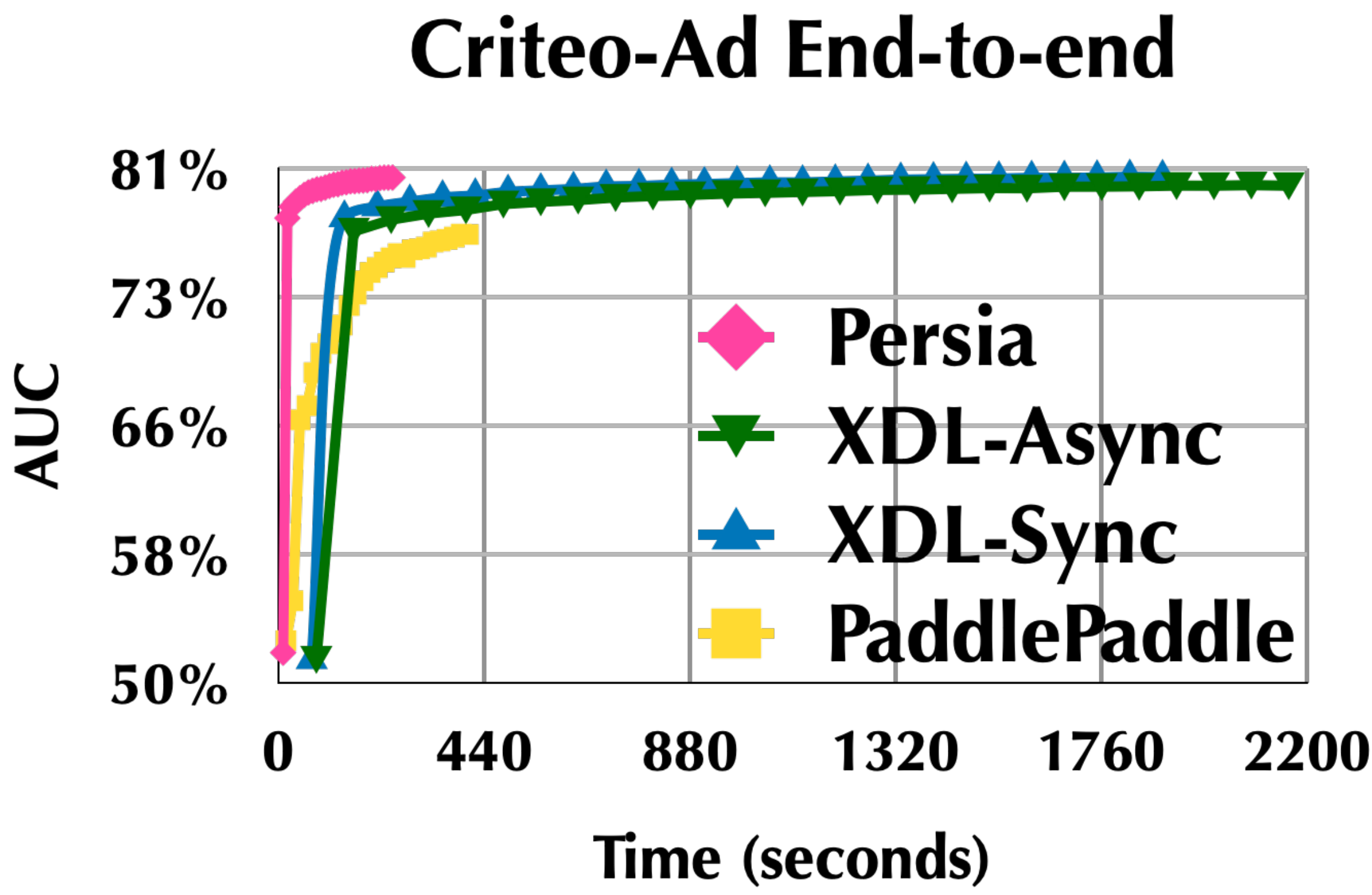}
    \end{subfigure}
    \begin{subfigure}
        \centering\includegraphics[width=13.35em]{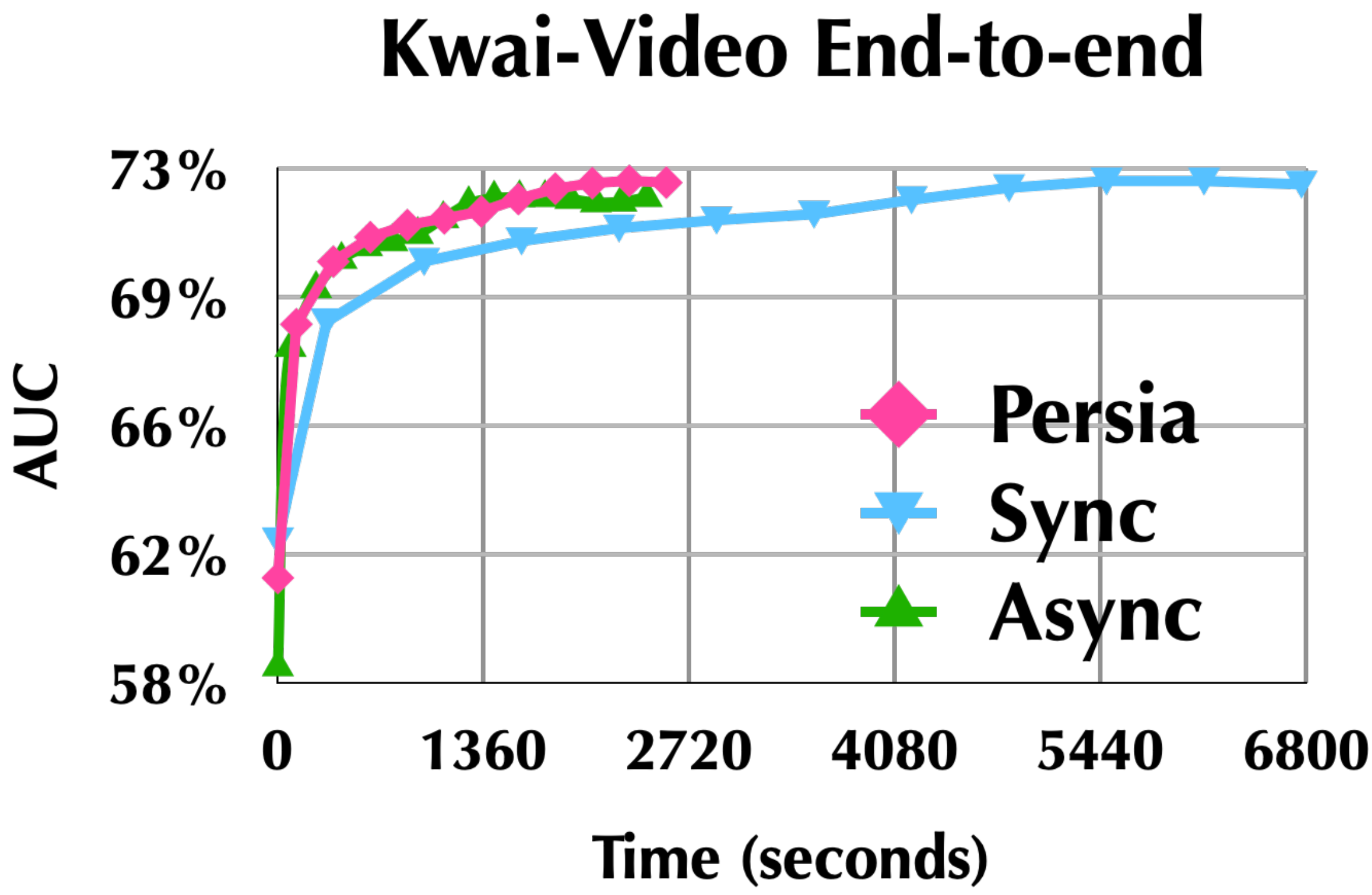}
    \end{subfigure}
    \vspace{-1.0em}
\caption{End-to-end training performance of four benchmarks: \texttt{Taobao-Ad}, \texttt{Avazu-Ad}, \texttt{Criteo-Ad}, \texttt{\kuaishou-Video} (from left to right). }
  \label{fig:exp_e2e}
\end{figure*}

\section{Experiments}
\label{sec:eval}

In this section, we introduce experiments to evaluate the design and implementation of \sys, focusing on the following questions:

\begin{itemize}[topsep=5pt, leftmargin=*]
    \item \textit{Can \sys provide significant boost of end-to-end training time comparing with the state-of-the-art systems? }
    \item \textit{Is the hybrid training algorithm statistically efficient in terms of convergence w.r.t training iterations?}
    \item \textit{Can \sys provide high training throughput and scale out well especially for trillion-parameter scale models?}
    
\end{itemize}

\textbf{Benchmark.} We evaluate \sys over three open-source benchmarks and one real-world production microvideo recommendation workflow at \kuaishou:

\begin{itemize}[topsep=5pt, leftmargin=*]
    \item \texttt{Taobao-Ad} (open source~\cite{taobaobenchmark}): predict the advertisement CTR from Taobao's website for 8 days with 26 million records.  
    
    \item \texttt{Avazu-Ad} (open source~\cite{avazubenchmark}): predict the advertisement CTR of Avazu's log for 11 days with 32 million records. 
    
    \item \texttt{Criteo-Ad} (open source~\cite{criteobenchmark}): predict the advertisement CTR of Criteo's traffic for 24 days with 44 million records. Note that we also extend this dataset (noted as \texttt{Criteo-Syn}) by synthesizing different number of random ID features for capacity and scalablity evaluation (see Section \ref{sec:exp_scal}). 
    
    \item \texttt{\kuaishou-Video} (confidential production dataset): predict the explicit behavior of \kuaishou's active users about the microvideo recommendation in $7$ days with 3 billion records.
\end{itemize}

For the three open source advertisement CTR benchmarks, we include $80\%$ of the records as training set and the rest $20\%$ of the records as test set, we consider a fully connected feed forward neural network (FFNN) as the deep learning model with five hidden layer dimensions of $4096$, $2048$, $1024$, $512$ and $256$. For the \kuaishou production microvideo recommendation task, $85\%$ of the data are included in the training set while the rest $15\%$ are considered as the test set, we also use FFNN as the model to predict multiple user behaviors. We report test AUC to evaluate convergence. The model scale of the each benchmark is listed in Table \ref{tab:scale}

\begin{table}[t!]
\centering
\begin{tabular}{|l|r|r|r}
\hline
               &Sparse \# parameter & Dense \# parameter  \\ 
\hline
\texttt{Taobao-Ad}      & 29 Million  & 12 Million  \\ 
\texttt{Avazu-Ad}       & 134 Million  & 12 Million  \\ 
\texttt{Criteo-Ad}       & 540 Million & 12 Million \\ 
\hline
\texttt{\kuaishou-Video} & 2 Trillion & 34 Million \\ 
\hline
\texttt{Criteo-Syn$_1$} &  6.25 Trillion & 12 Million     \\
\texttt{Criteo-Syn$_2$} &  12.5 Trillion & 12 Million     \\
\texttt{Criteo-Syn$_3$} &  25 Trillion & 12 Million     \\
\texttt{Criteo-Syn$_4$} &  50 Trillion & 12 Million     \\
\texttt{Criteo-Syn$_5$} &  \textbf{100 Trillion} & 12 Million     \\
\hline
\end{tabular}
\caption{Model scales for the benchmarks.}
\label{tab:scale}
\vspace{-2.5em}
\end{table}

\vspace{0.5em}
\textbf{Baseline systems.} We consider two state-of-the-art baselines: \textsc{XDL} \cite{jiang2019xdl} and \textsc{PaddlePaddle} \cite{paddle}---\textsc{XDL} is a specialized distributed recommendation framework developed by Alibaba; \textsc{PaddlePaddle} is a general purpose deep learning framework from Baidu with a special \texttt{Heter} mode that implements the design of AIBox \cite{zhao2019aibox}, Baidu's recommendation training system, according to private communications we had with 
members of the \textsc{PaddlePaddle} development community.

\vspace{0.5em}
\textbf{Cluster setup.}
We conduct experiments on two clusters.
Most of the training is conducted over heterogeneous clusters inside \kuaishou's production data center---for \sys, we include up to $64$ Nvidia V100 GPUs, and $100$ CPU instances (each with $52$ cores and $480$GB RAM).
The instances are connected by a network with the bandwidth of 100 Gbps. The baseline systems (\textsc{XDL} and \textsc{PaddlePaddle}) are equipped with the same amount of computation resources for each individual setting. 

We further conduct larger-scale scalability experiments over Google cloud platform,
for both capacity and \textit{public reproducibility} 
reasons\footnote{We appreciate Google's great support for coordinating the computation resources for this experiment.} with a heterogeneous cluster including:

\vspace{0.5em}
\begin{itemize}[topsep=0pt, leftmargin=*]
\item \underline{8 a2-highgpu-8g} instances (each with 8 Nvidia A100 GPUs) as NN workers;
\item \underline{100 c2-standard-30} instances (each with 30vCPUs, 120GB RAM) as embedding workers;
\item \underline{30 m2-ultramem-416} instances (each with 416vCPUs, 12TB RAM) as embedding PS.
\end{itemize}
\vspace{0.5em}

\vspace{-1.1em}
\subsection{End-to-end Evaluation}
We first compare the end-to-end training time that 
each system needs to achieve a given AUC,
using the three open source benchmarks over a heterogeneous cluster with $8$ GPU workers. We report the results of both synchronous and asynchronous modes of \textsc{XDL}, we conduct a careful trial of different modes in \textsc{PaddlePaddle} and report the result under the optimal setting. Figure \ref{fig:exp_e2e} illustrates significant performance improvements from \sys: e.g., for the \texttt{Taobao-Ad} benchmark, \sys is $7.12\times$ and $8.4\times$ faster than that of the synchronous and asynchronous modes of \textsc{XDL}, and $1.71\times$ faster than \textsc{PaddlePaddle}--same level of speedup also appears in the \texttt{Avazu-Ad} and \texttt{Criteo-Ad} benchmark. 

On our production \texttt{\kuaishou-Video} dataset, 
\textsc{XDL} fails to run the task to convergence
in $100$ hours---we compare the training sample throughput and find that \sys's throughput is $19.31\times$ of \textsc{XDL}). 
For \textsc{PaddlePaddle}, we fail to run the training task since \textsc{PaddlePaddle} does not support some deep learning operators (e.g., batch normalization) required in our model.

\vspace{-0.5em}
\subsection{Evaluation of Hybrid Algorithm}
\label{sec:exp_alg}

\begin{figure*}[t!]
\centering
    \begin{subfigure}
       \centering\includegraphics[width=13.35em]{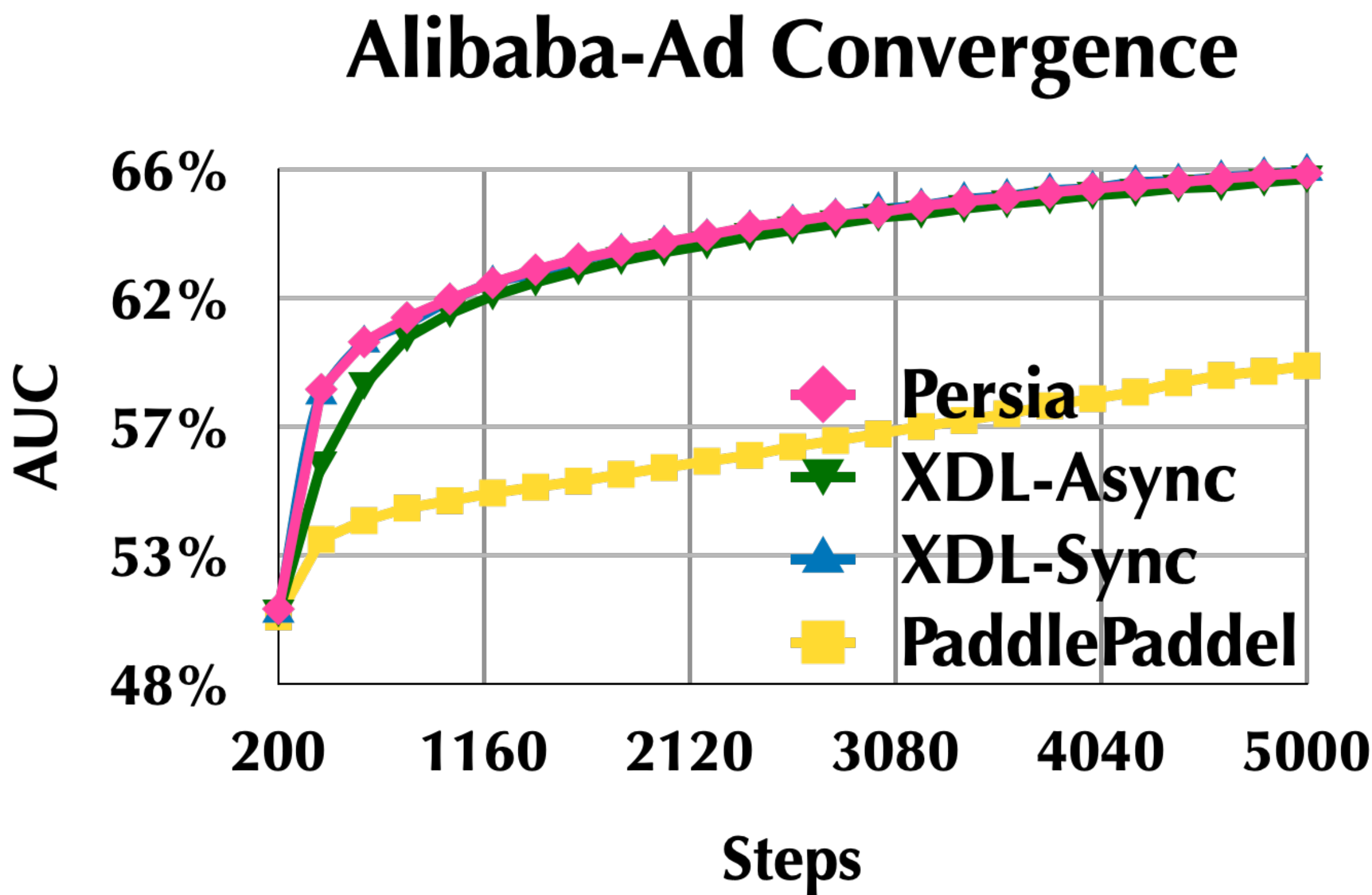}
    \end{subfigure}
    \begin{subfigure}
        \centering\includegraphics[width=13.35em]{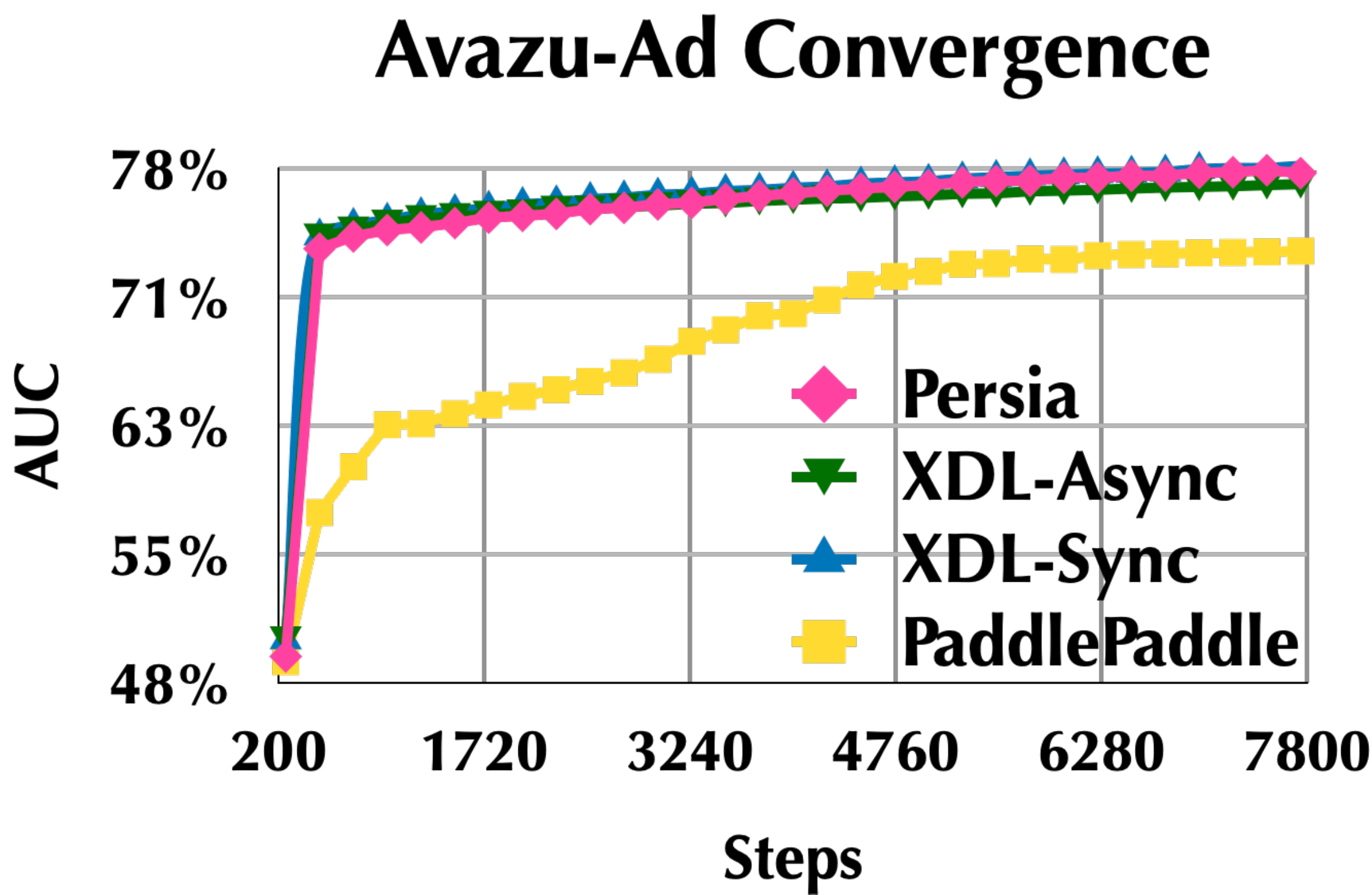}
    \end{subfigure}
    \begin{subfigure}
        \centering\includegraphics[width=13.35em]{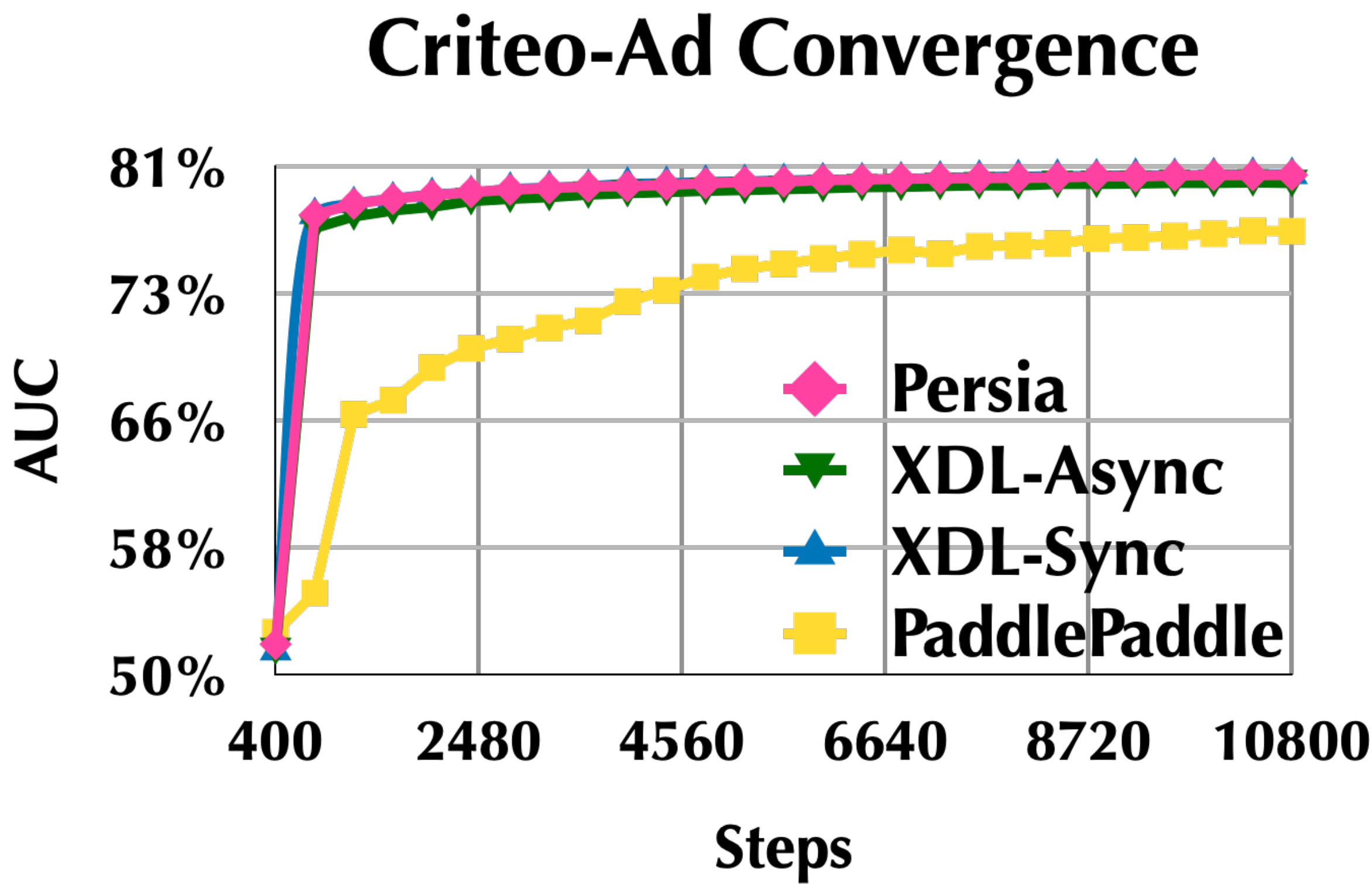}
    \end{subfigure}
    \begin{subfigure}
        \centering\includegraphics[width=13.35em]{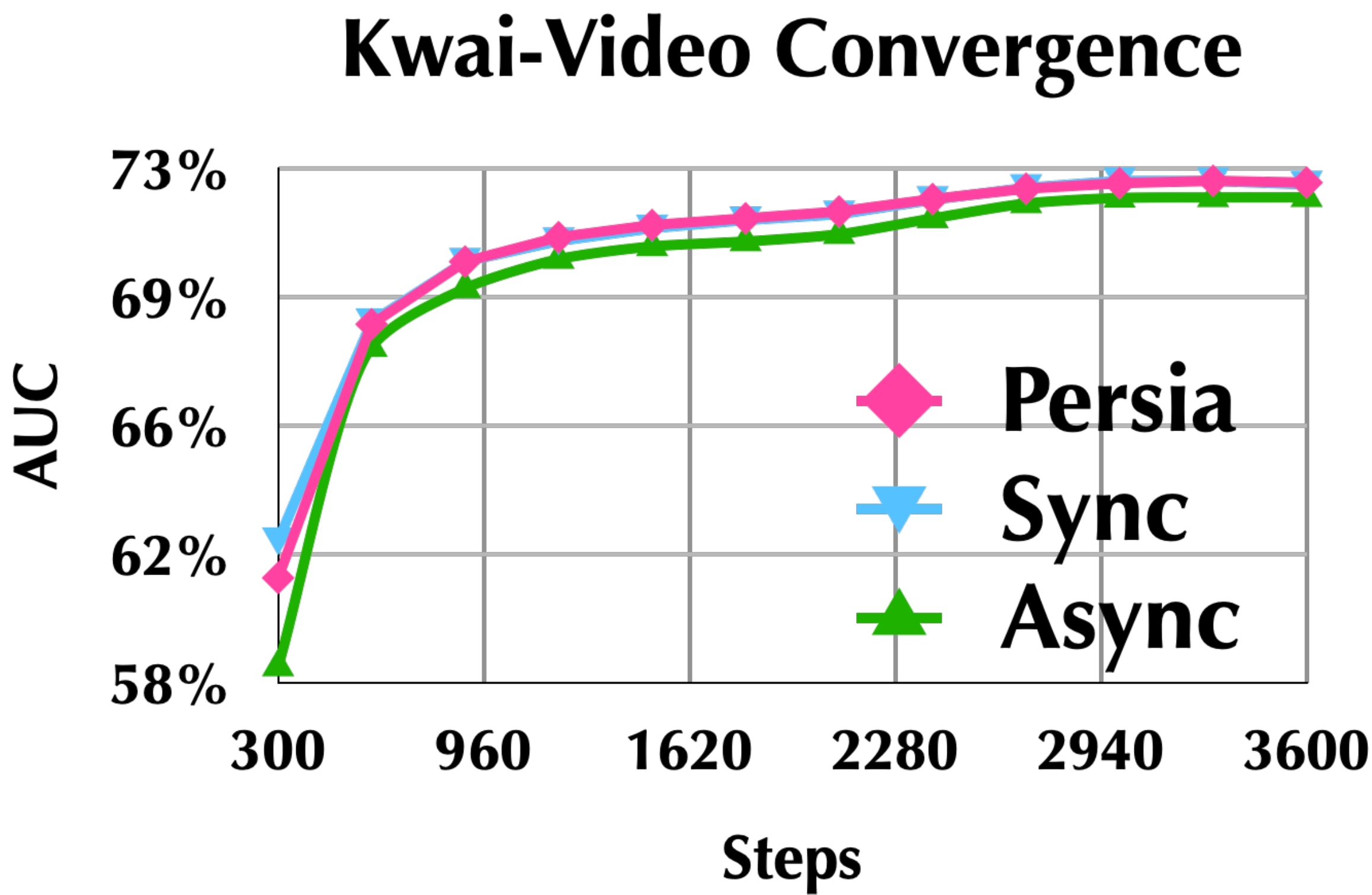}
    \end{subfigure}
    \vspace{-1em}
\caption{Convergence of four benchmarks: \texttt{Taobao-Ad}, \texttt{Avazu-Ad}, \texttt{Criteo-Ad}, \texttt{\kuaishou-Video}.} 
  \label{fig:exp_converge}
\end{figure*}

\begin{table*}[t!]
\centering
\begin{tabular}{l|c|c|c|c}
\hline
               & \sys-Hybrid  & Sync & Async & \textsc{PaddlePaddle} \\ \hline
\texttt{Taobao-Ad}      & 65.86\%  & 65.93\% (XDL)  & 65.66\% (XDL)   & 59.14\%   \\ 
\texttt{Avazu-Ad}       & 77.93\% & 78.02\%  (XDL) & 77.10\% (XDL)  & 73.19\%      \\
\texttt{Criteo-Ad}      & 80.47\% & 80.51\% (XDL) & 80.01\% (XDL)  & 77.08\%       \\ 
\texttt{\kuaishou-Video} & 72.63\% & 72.63\%    & 72.16\% & -      \\
\hline
\end{tabular}
\caption{Final test AUC of four benchmarks.}
\label{tab:auc}
\vspace{-2.0em}
\end{table*}

The hybrid algorithm is one of the main reasons behind \sys's scalability.
Here, we evaluate its convergence behavior and leave it scalability 
and performance to the next section. Figure~\ref{fig:exp_converge}
illustrates the convergence behaviors of different systems.
We see that the hybrid algorithm shows almost identical convergence when comparing with the fully synchronous mode, converging to 
comparable AUC as illustrated in Table~\ref{tab:auc}
(but as we will see in the next section, the hybrid algorithm is much faster compared with the fully synchronous algorithm).
We see that test AUC gap between the hybrid mode and synchronous mode is always less than $0.1\%$ in the three open-source benchmarks, and less than $0.001\%$ in the production \texttt{\kuaishou-video} benchmark; by contrast, the gap between the asynchronous mode and the synchronous mode is much higher (from $0.5\%$ to $1.0\%$); further, as we allow more aggressive asynchronicity in \textsc{PaddlePaddle}, the gap is more significant.
We want to emphasize that even $0.1\%$ decrease of accuracy for the production recommender model might lead to a loss of revenue at the scale of a million dollar---this is \textit{NOT acceptable}.

\begin{figure*}[t!]
\centering
    \begin{subfigure}
        \centering\includegraphics[width=13.35em]{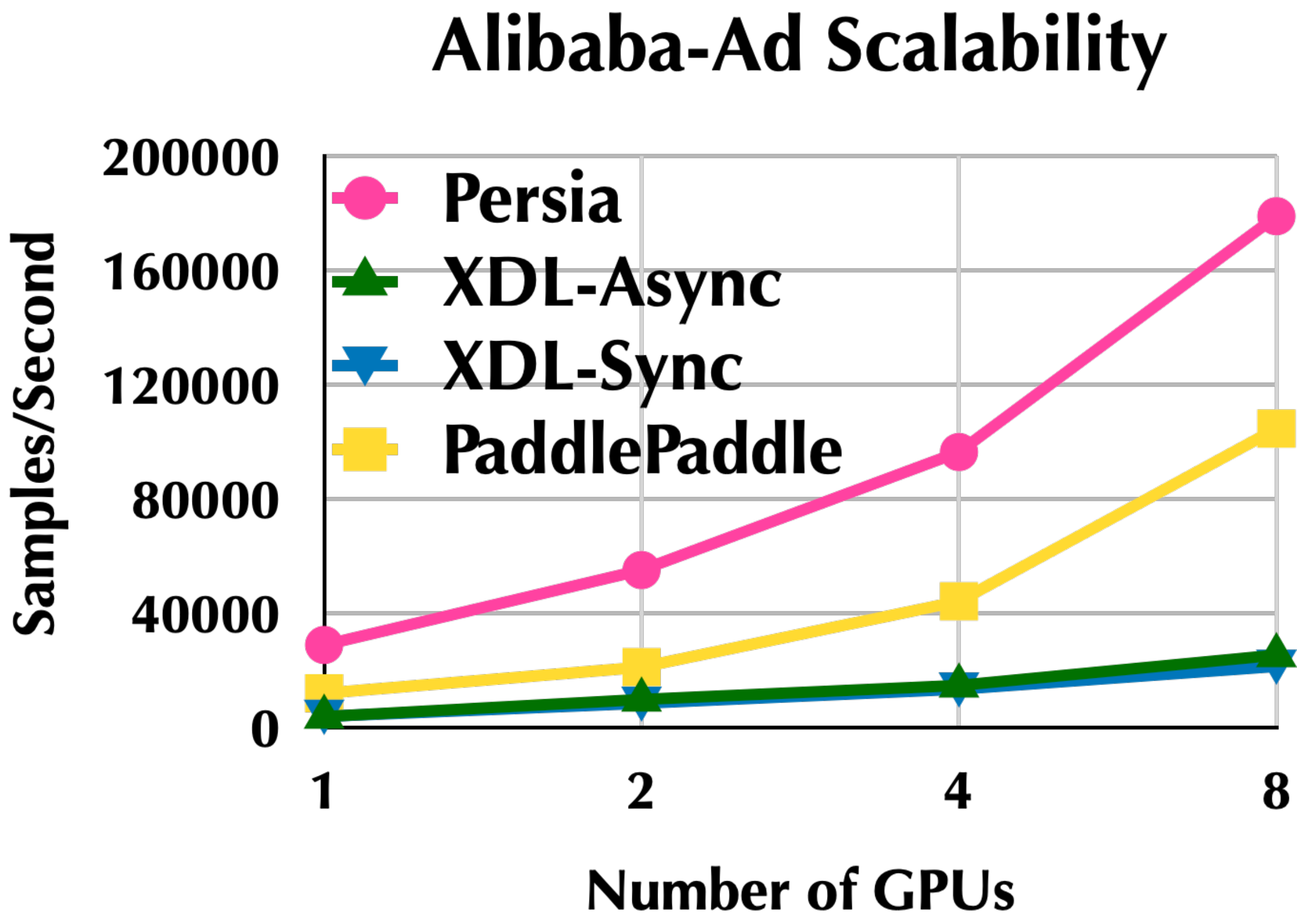}
    \end{subfigure}
    \begin{subfigure}
        \centering\includegraphics[width=13.35em]{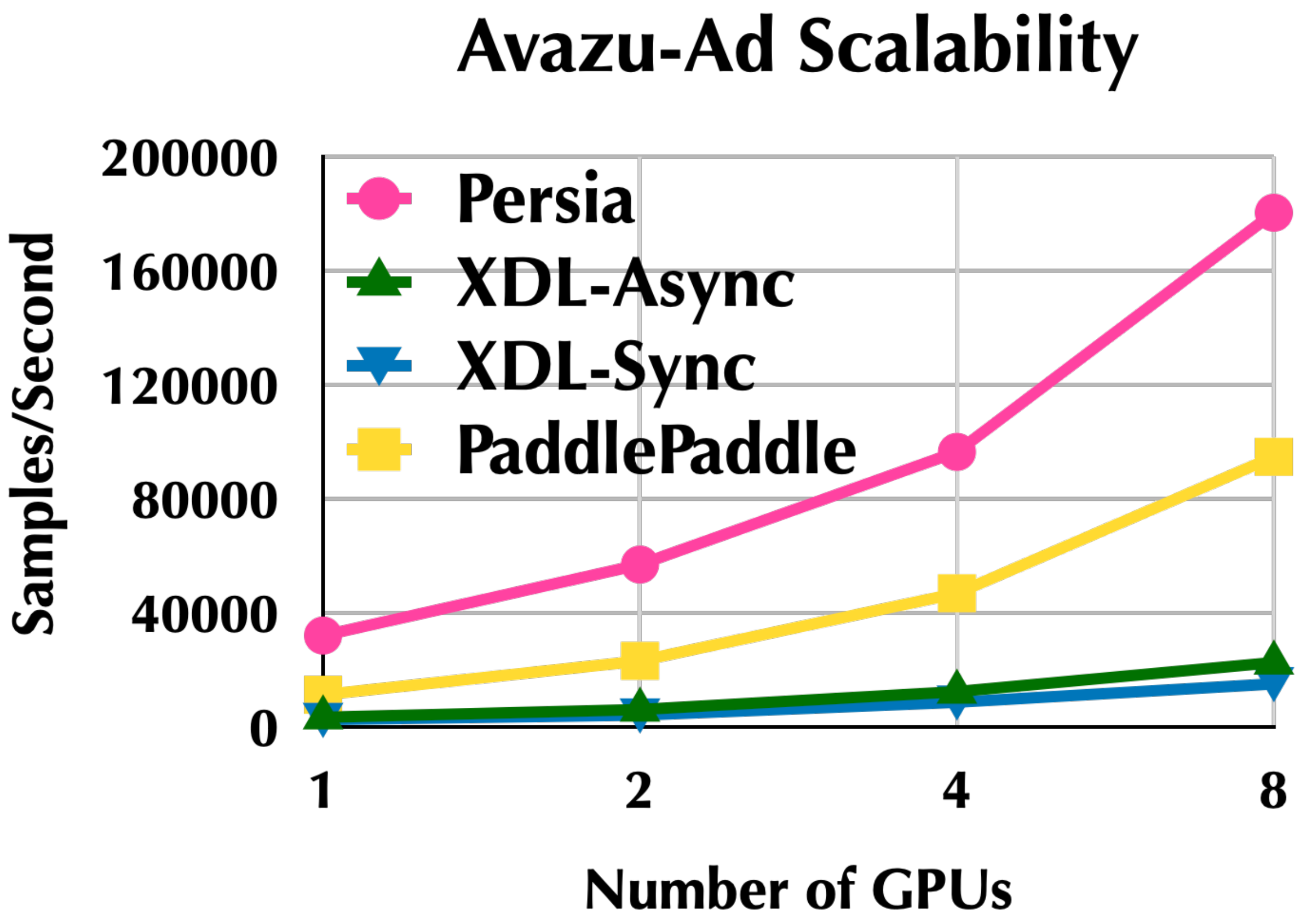}
    \end{subfigure}
    \begin{subfigure}
        \centering\includegraphics[width=13.35em]{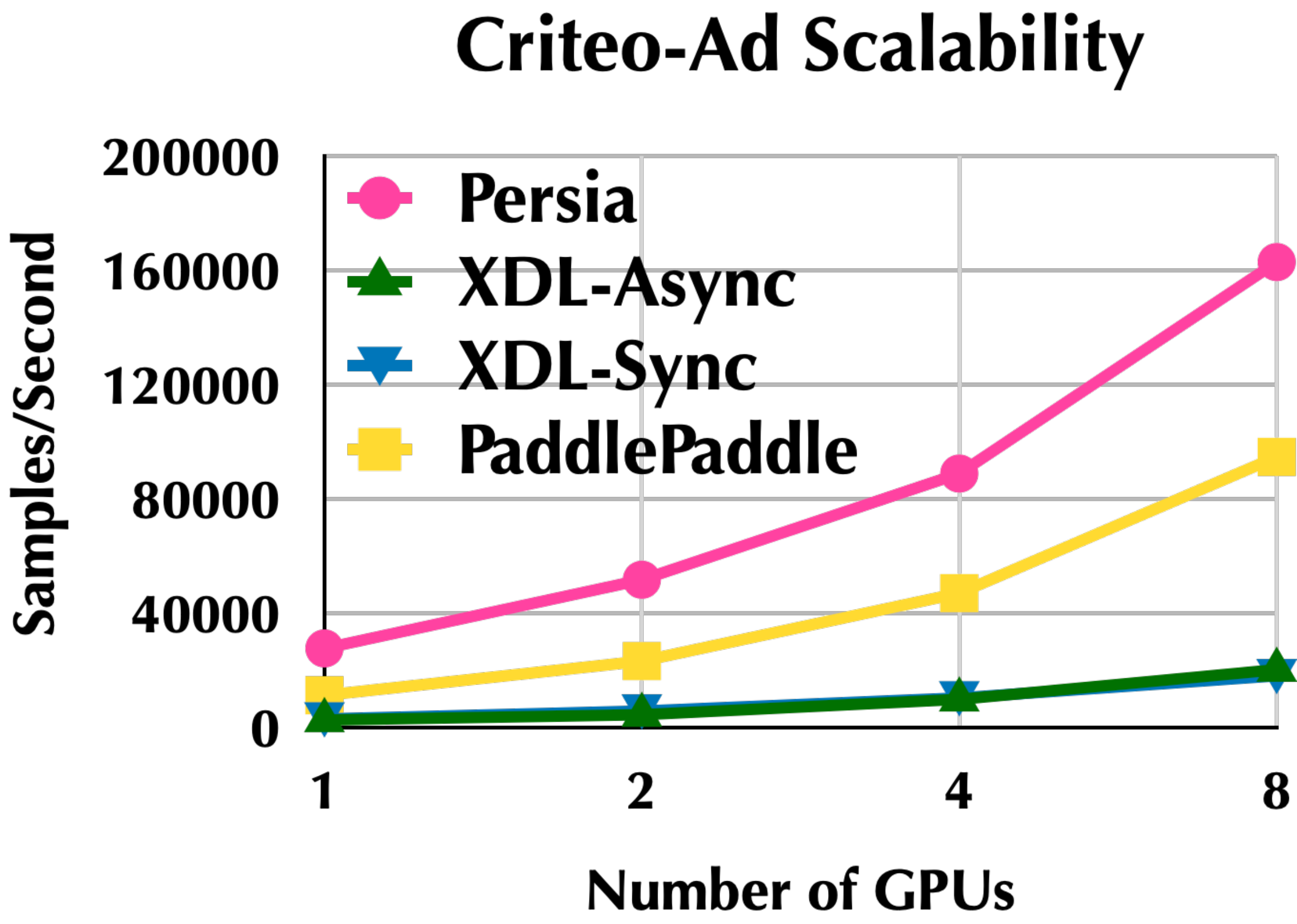}
    \end{subfigure}
    \begin{subfigure}
        \centering\includegraphics[width=13.35em]{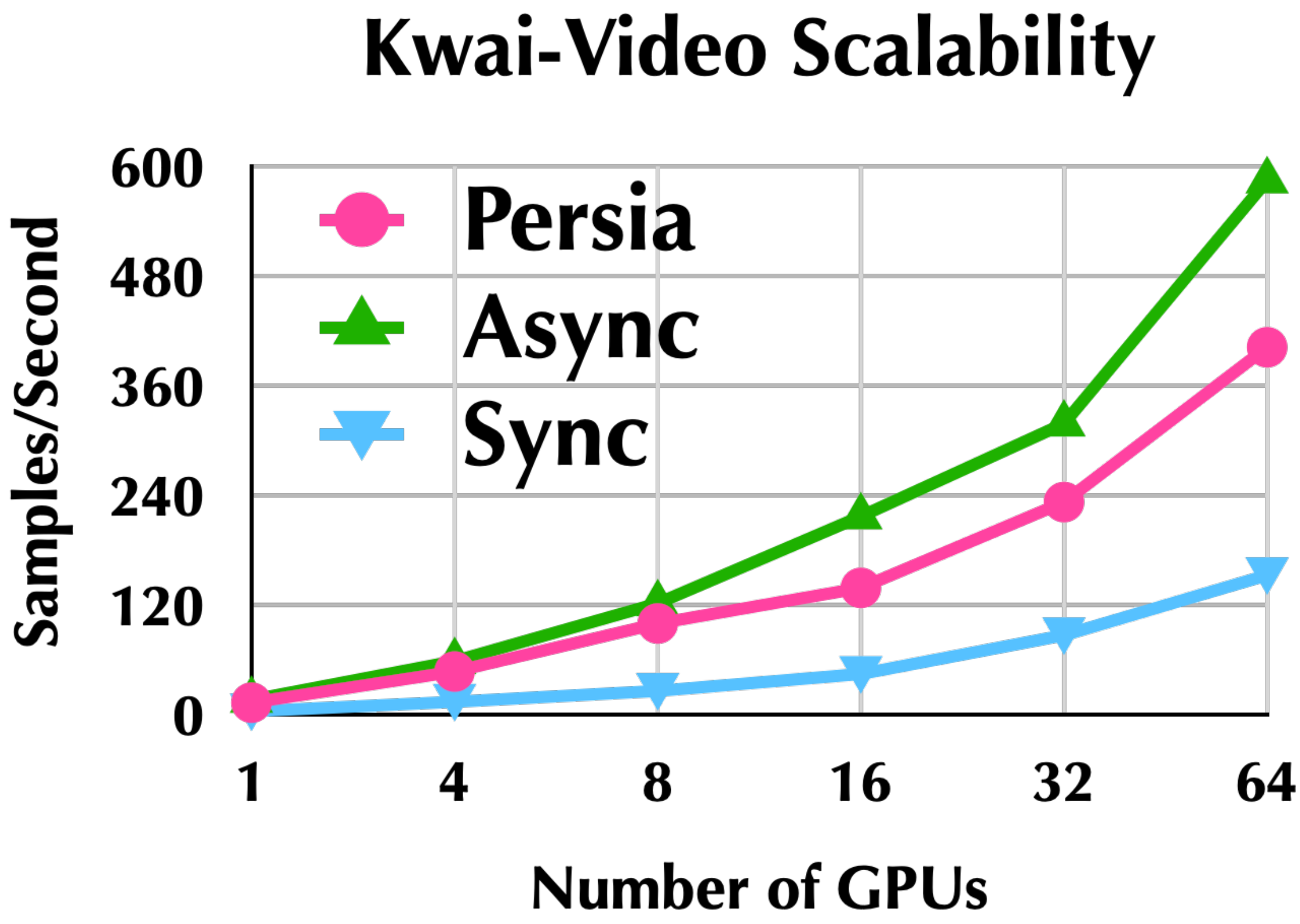}
    \end{subfigure}
    \vspace{-0.5em}
  \caption{Scalability and performance of four benchmarks: \texttt{Taobao-Ad}, \texttt{Avazu-Ad}, \texttt{Criteo-Ad}, \texttt{\kuaishou-Video} (from left to right). }
  \vspace{-0.5em}
  \label{fig:exp_scale}
\end{figure*}

\begin{figure}[t!]
    \centering\includegraphics[width=.47\textwidth]{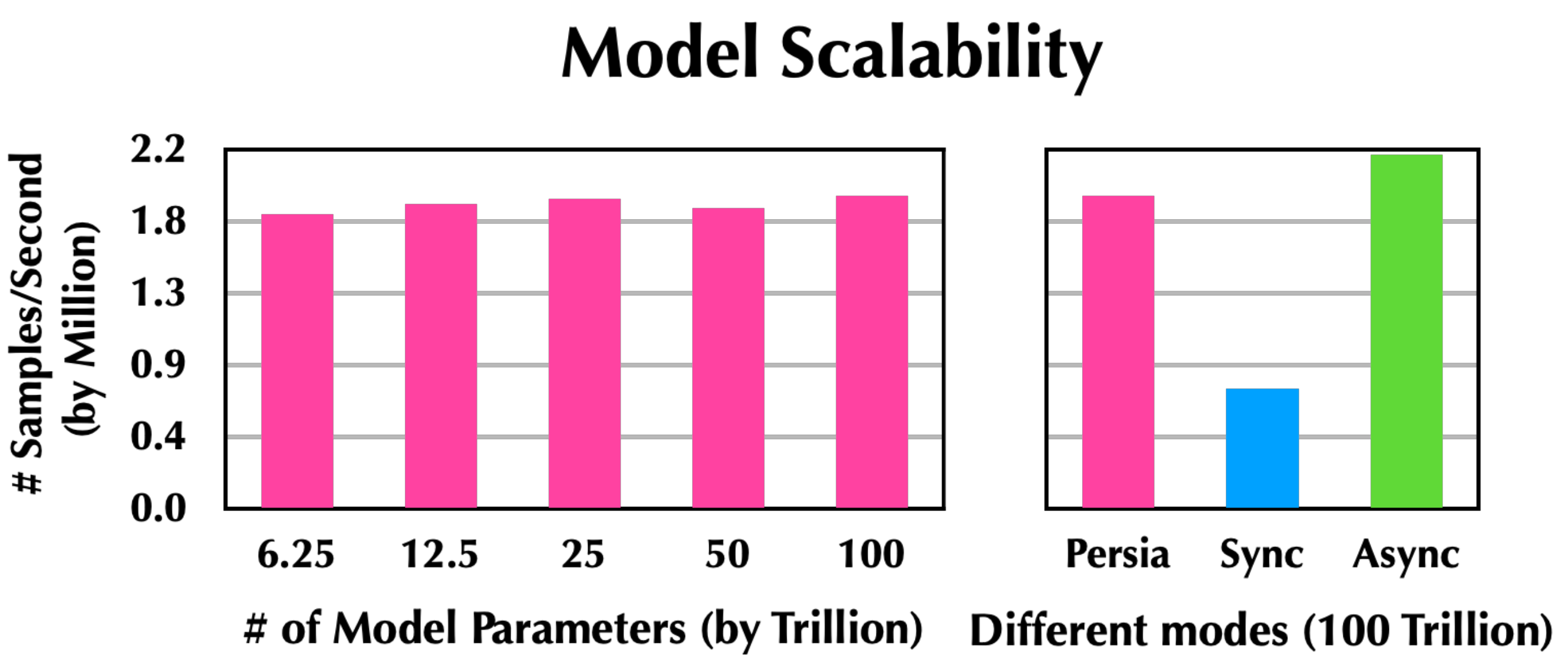}
    \vspace{-1.0em}
    \caption{Capacity test of \sys for \texttt{Criteo-Syn} benchmark the over Google cloud platform.}
    \label{fig:capacity}
    \vspace{-1.0em}
\end{figure}

\vspace{-0.5em}
\subsection{Scalability and Capacity}
\label{sec:exp_scal}

We first evaluate the scalability of \sys over clusters with \textit{different scales of GPUs} in terms of training sample throughput, as illustrated in Figure~\ref{fig:exp_scale}.
We see that \sys, with a hybrid algorithm, achieves much higher throughput compared to all other systems. 
On three open source benchmarks, \sys reaches nearly linear speedup with significantly higher throughput comparing with \textsc{XDL} and \textsc{PaddlePaddle}.
For the \texttt{\kuaishou-video} benchmark, \sys achieves $3.8\times$ higher throughput compared with the fully synchronous algorithm.
Note that the fully asynchronous algorithm can achieve faster throughput than the hybrid algorithm; however, as illustrated in the previous section, it incurs a lower AUC compared with the hybrid algorithm.

We further investigate the capacity of \sys for training \textit{different scales of trillion-parameter models} in Google cloud platform. Figure \ref{fig:capacity}-left illustrates that the throughput of \sys when varying the number of entrances in the embedding table and fixing the embedding output dimension to $128$---this is corresponding to the model scale of \texttt{Criteo-Syn} in Table \ref{tab:scale}. We see that \sys shows stable training throughput when increasing the model size even up to 100 trillion parameters. This can be viewed as concrete evidences that \sys has the capacity to scale out to the largest recommender model that is never reported before. Figure \ref{fig:capacity}-right shows that for the 100 trillion-parameter model, \sys also achieves $2.6\times$ higher throughput than the fully synchronous mode; on the other hand, asynchronous mode introduces further speedup ($1.2\times$ faster than that of the hybrid algorithm)---this generally would introduce statistical inefficiency for convergence, which is not revealed by the capacity evaluation.

\section{Related Work}
\label{sec:rel}

We briefly discuss the relevant recommender systems and system relaxations for distributed large scale learning. Detailed discussion can be found in more comprehensive surveys (e.g., \cite{zhang2019deep} for recommender system, and \cite{liu2020distributed} for distributed deep learning).

\subsection{Recommender Systems}
Recommender systems are critical tools to enhance user experience and promote sales and services \cite{zhang2019deep}, where deep learning approaches have shown great advance recently and been deployed ubiquitously by tech companies such as Alibaba~\cite{zhou2019deep}, Amazon~\cite{raman2019scaling}, Baidu~\cite{zhao2020distributed}, Facebook~\cite{naumov2020deep}, Google~\cite{covington2016deep, zhao2019recommending}, Kuaishou~\cite{xie2020kraken}, Pinterest~\cite{liu2017related}, Netflix~\cite{gomez2015netflix}, etc. Generally, deep learning based approaches take the sparse ID type features to generate dense vectors through an embedding layer, and then feed the embedding activations into the proceeding neural components to expose the correlation and interaction.
Diversified features are included to power recommender systems~\cite{zhang2017joint}, such as sequential information (e.g., session based user id~\cite{tan2016improved, twardowski2016modelling, tuan20173d}), text~\cite{zuo2016tag,zhang2017hashtag,zhao2018recommendations}, image~\cite{wang2017your,wen2018visual}, audio~\cite{van2013deep,wang2014improving}, video~\cite{chen2017attentive,lee2018collaborative}, social network~\cite{hsieh2016immersive,deng2016deep}, etc.
Additionally, various neural components have been explored for recommender systems, including multi-layer perceptron (MLP)~\cite{xu2016tag,yang2017bridging}, convolutional neural network (CNN)~\cite{seo2017interpretable,tuan20173d}, recurrent neural network (RNN)~\cite{wu2016recurrent,wu2016personal}, Autoencoder~\cite{sedhain2015autorec,zhuang2017representation}, and deep reinforcement learning~\cite{zhao2018recommendations,zheng2018drn}. Furthermore, multiple relevant goals for recommendation can be learned simultaneously by multi-task learning~\cite{covington2016deep,zhou2018deep,zhou2019deep,zhang2020model,naumov2020deep,zhao2020autoemb}.

Due to the massive scale of both training data and recommender models, distributed learning techniques are applied to train industrial-scale recommender systems~\cite{zhou2019deep,raman2019scaling,zhao2020distributed,zhao2019recommending,naumov2020deep}. The heterogeneity in computation naturally leads to the idea of utilizing heterogeneous distributed computation resources to handle such training procedure~\cite{jiang2019xdl,raman2019scaling,zhao2019aibox,pan2020dissecting,zhao2020distributed,naumov2020deep,krishna2020accelerating,wilkening2021recssd,scaling2022het}.
For example, XDL \cite{jiang2019xdl} adopts the advanced model server to replace classic parameter server to reduce the traffic between CPU nodes and GPU nodes; 
Baidu proposes a hierarchical PS architecture~\cite{zhao2020distributed} to implement a sophisticated caching schema to reduce the communication overhead;
HET~\cite{scaling2022het} adopts an advanced cache mechanism at the GPU worker side to leverage the skewness of embeddings by caching frequently updated entrances in the worker's limited local memory.  
However, as far as we know, there is a lack of a carefully designed distributed training solution to support the heterogeneity inherited from the training data (sparse vs. dense) and gradient decent based optimization (memory bounded vs. computation bounded) over mixed hardwares and infrastructures (CPU instances vs. GPU instances).

\subsection{Distributed Deep Learning}
Distributed deep learning can be categorized as data parallelism and model parallelism. 
In \textit{data parallel} distributed deep learning, two main categories of systems are designed---parameter server~\cite{dean2012large, li2014scaling,FlexpsVLDB,HeteroSIGMOD,PS2} and AllReduce~ \cite{sergeev2018horovod,zhang2019mllib,jiang2018dimboost,chen2016xgboost}.
In a PS architecture, models are stored in a single node or distributively in multiple nodes; during the training phase, workers periodically fetch the model from PS, conduct forward/backward propagation, and push the gradients to the PS, while the PS aggregates the gradients and updates the parameters.
With an AllReduce paradigm, all workers collaborate with their neighbors for model/gradient exchanges.
Different strategies are proposed to speed up the expensive parameter/gradient exchange phase. To reduce communication volumes, lossy communication compression methods are introduced, such as quantization~\cite{alistarh2016qsgd,zhang2017zipml,bernstein2018signsgd,wen2017terngrad},
sparsification~\cite{wangni2018gradient,alistarh2018convergence,wang2018atomo,wang2017efficient}, sketching~\cite{ivkin2019communication}, and
error compensation~\cite{tang2019doublesqueeze}).
In an attempt to get rid of the latency bottleneck, decentralized communication approaches are proposed~\cite{koloskova2019decentralized,li2018pipe,lian2017can,lian2018asynchronous,tang2018d}. Additionally, local SGD is discussed to optimize for the number of communication rounds during training~\cite{wang2019adaptive,lin2019don,stich2018local,haddadpour2019local}.
To remove the synchronization barrier, asynchronous update methods are proposed~\cite{lian2015asynchronous,peng2017asynchronous,zheng2017asynchronous,zhou2018distributed,simsekli2018asynchronous,nguyen2018sgd}. There are also approaches that combines multiple strategies listed above~\cite{lian2018asynchronous,basu2019qsparse,koloskova2019decentralized,tang2019deepsqueeze,beznosikov2020biased}.
On the other hand, researches about \textit{model parallelism} attempt to study how to allocate model parameters and training computation across compute units in a cluster to maximize training throughput and minimize communication overheads. Optimizations are proposed for both operation partitioning approach~\cite{jia2019beyond,wang2019supporting,shazeer2018mesh,shoeybi2019megatron} and pipeline parallel approach~\cite{huang2019gpipe,harlap2018pipedream,tarnawski2020efficient}.
Recently, there are also approaches that combine both data and model parallelism~\cite{raman2019scaling, park2020hetpipe, li2021terapipe}.

General purpose deep learning systems have been one of the main driving forces behind the rapid advancement of machine learning techniques with the increasing scalability and performance 
for distributed deep learning. Popular options include TensorFlow~\cite{abadi2016tensorflow}, PyTorch~\cite{pytorch}, MXNet~\cite{chen2015mxnet}, PaddlePaddle~\cite{paddle}, MindSpore~\cite{mindspore}, etc.
Extensions and modifications have been made based on these general purpose learning systems for efficient distributed learning (e.g., Horovod~\cite{sergeev2018horovod}, BytePS~\cite{jiang2020unified}, Bagua~\cite{gan2021bagua}, Megatron~\cite{shoeybi2019megatron}, ZeRO~\cite{rajbhandari2020zero}, SageMaker~\cite{karakus2021amazon}, etc.). However, even including these extensions, the current general purpose deep learning systems do not consider the challenges about handling the heterogeneity over a hybrid infrastructure. Thus, it is difficult to directly use these general purpose learning systems for the training tasks of industrial-scale recommender systems over a heterogeneous cluster.

\section{Conclusion}
\label{sec:con}
We introduce \sys, a distributed system to support efficient and scalable training of recommender models at the scale of 100 trillion parameters. 
We archive the statistical and hardware efficiency by a careful co-design of both the distributed training algorithm and the distributed training system. 
Our proposed hybrid distributed training algorithm introduces elaborate system relaxations for efficient utilization of heterogeneous clusters, while converging as fast as vanilla SGD. 
We implement a wide range of system optimizations to overcome the unique challenges raised by the hybrid algorithmic design. 
We evaluate \sys using both publicly available benchmark tasks and production tasks at \kuaishou. 
We show that \sys leads to up to $7.12\times$ speedup compared to alternative state-of-the-art approaches; additionally, \sys can also scale out to 100-trillion-parameter model on Google cloud platform.

\sys has been released as an open source project on github with concrete instructions about setup over Google cloud platform---we hope everyone from both academia and industry would find it easy to train 100-trillion-parameter scale deep learning recommender models. 



\balance
\bibliographystyle{ACM-Reference-Format}
\bibliography{persia}

\appendix

\newpage
\onecolumn
\section*{Appendix}

\begin{lemma}
  \label{lemma:ajksdf}
  Under Assumption \ref{assumption:main}, for any $\delta \in \mathbb{R}^{\embDimension}$ the following inequality holds:
\begin{align*}
   \left\| f' \left(\embWeight + \delta, \nnWeight\right) - f' \left(\embWeight, \nnWeight\right) \right\|
  \leqslant \alpha L \left\| \delta \right\|,\quad \forall \embWeight\in \mathbb{R}^{\embDimension}, \nnWeight \in \mathbb{R}^{\nnDimension}. \numberthis \label{eq:crack-skylark-1}
\end{align*}
\end{lemma}

\begin{proof}
Define Hessian matrix w.r.t. the sparse weights $\mathbf{w}^{\text{emb}}$:
\begin{align*}
  \mathbf{H}_{F_{\xi}}^{\text{emb}} \assign & \left(\begin{array}{cccc}
    \frac{\partial^2 F \left( \mathbf{w}^{\text{emb}},
    \mathbf{w}^{\text{nn}} ; \xi \right)}{\partial \left[
    \mathbf{w}^{\text{emb}}_1 \right]^2} & \frac{\partial^2 F \left(
    \mathbf{w}^{\text{emb}}, \mathbf{w}^{\text{nn}} ; \xi
    \right)}{\partial \mathbf{w}^{\text{emb}}_1 \partial
    \mathbf{w}^{\text{emb}}_2} & \cdots & \frac{\partial^2 F \left(
    \mathbf{w}^{\text{emb}}, \mathbf{w}^{\text{nn}} ; \xi
    \right)}{\partial \mathbf{w}^{\text{emb}}_1 \partial
    \mathbf{w}^{\text{emb}}_{N^{\text{emb}}}}\\
    \frac{\partial^2 F \left( \mathbf{w}^{\text{emb}},
    \mathbf{w}^{\text{nn}} ; \xi \right)}{\partial
    \mathbf{w}^{\text{emb}}_2 \partial \mathbf{w}^{\text{emb}}_1} &
    \frac{\partial^2 F \left( \mathbf{w}^{\text{emb}},
    \mathbf{w}^{\text{nn}} ; \xi \right)}{\partial \left[
    \mathbf{w}^{\text{emb}}_2 \right]^2} & \cdots & \frac{\partial^2 F
    \left( \mathbf{w}^{\text{emb}}, \mathbf{w}^{\text{nn}} ; \xi
    \right)}{\partial \mathbf{w}^{\text{emb}}_2 \partial
    \mathbf{w}^{\text{emb}}_{N^{\text{emb}}}}\\
    \vdots & \vdots & \ddots & \vdots\\
    \frac{\partial^2 F \left( \mathbf{w}^{\text{emb}},
    \mathbf{w}^{\text{nn}} ; \xi \right)}{\partial
    \mathbf{w}^{\text{emb}}_n \partial \mathbf{w}^{\text{emb}}_1} &
    \frac{\partial^2 F \left( \mathbf{w}^{\text{emb}},
    \mathbf{w}^{\text{nn}} ; \xi \right)}{\partial
    \mathbf{w}^{\text{emb}}_n \partial \mathbf{w}^{\text{emb}}_2} &
    \cdots & \frac{\partial^2 F \left( \mathbf{w}^{\text{emb}},
    \mathbf{w}^{\text{nn}} ; \xi \right)}{\partial \left[
    \mathbf{w}^{\text{emb}}_n \right]^2}
  \end{array}\right),\\
  \mathbf{H}_f^{\text{emb}} \assign & \left(\begin{array}{cccc}
    \frac{\partial^2 f \left( \mathbf{w}^{\text{emb}},
    \mathbf{w}^{\text{nn}} \right)}{\partial \left[
    \mathbf{w}^{\text{emb}}_1 \right]^2} & \frac{\partial^2 f \left(
    \mathbf{w}^{\text{emb}}, \mathbf{w}^{\text{nn}}
    \right)}{\partial \mathbf{w}^{\text{emb}}_1 \partial
    \mathbf{w}^{\text{emb}}_2} & \cdots & \frac{\partial^2 f \left(
    \mathbf{w}^{\text{emb}}, \mathbf{w}^{\text{nn}}
    \right)}{\partial \mathbf{w}^{\text{emb}}_1 \partial
    \mathbf{w}^{\text{emb}}_{N^{\text{emb}}}}\\
    \frac{\partial^2 f \left( \mathbf{w}^{\text{emb}},
    \mathbf{w}^{\text{nn}} \right)}{\partial
    \mathbf{w}^{\text{emb}}_2 \partial \mathbf{w}^{\text{emb}}_1} &
    \frac{\partial^2 f \left( \mathbf{w}^{\text{emb}},
    \mathbf{w}^{\text{nn}} \right)}{\partial \left[
    \mathbf{w}^{\text{emb}}_2 \right]^2} & \cdots & \frac{\partial^2 f
    \left( \mathbf{w}^{\text{emb}}, \mathbf{w}^{\text{nn}}
    \right)}{\partial \mathbf{w}^{\text{emb}}_2 \partial
    \mathbf{w}^{\text{emb}}_{N^{\text{emb}}}}\\
    \vdots & \vdots & \ddots & \vdots\\
    \frac{\partial^2 f \left( \mathbf{w}^{\text{emb}},
    \mathbf{w}^{\text{nn}} \right)}{\partial
    \mathbf{w}^{\text{emb}}_n \partial \mathbf{w}^{\text{emb}}_1} &
    \frac{\partial^2 f \left( \mathbf{w}^{\text{emb}},
    \mathbf{w}^{\text{nn}} \right)}{\partial
    \mathbf{w}^{\text{emb}}_n \partial \mathbf{w}^{\text{emb}}_2} &
    \cdots & \frac{\partial^2 f \left( \mathbf{w}^{\text{emb}},
    \mathbf{w}^{\text{nn}} \right)}{\partial \left[
    \mathbf{w}^{\text{emb}}_n \right]^2}
  \end{array}\right) .
\end{align*}

From the Lipschitzian assumption \eqref{eq:able-snipe} and the stochastic
gradient function $F'(\cdot; \xi)$ is differentiable, we obtain the Hessian matrix of $F(\cdot;\xi)$
is $\preccurlyeq L I$, where $I$ is the identity matrix. Noting that since for
each sample, the $\mathbf{w}^{\text{emb}}$ is used by a lookup operation
\eqref{eq:activation-set-def} where only a subset of the weights get used. For
any element $\mathbf{w}^{\text{emb}}_i$, if $i\not\in \Omega_{\xi}$,
$\mathbf{w}^{\text{emb}}_i$ will have no effect on the value of
$F'(\cdot; \xi)$. This implies the corresponding elements are zero in
$F(\cdot;\xi)$'s Hessian matrix, and the following inequality holds:
\begin{align*}
  \mathbf{H}_{F_{\xi}}^{\text{emb}} \preccurlyeq & LI_{\xi},
\end{align*}
where
\begin{align*}
  (I_{\xi})_{i, j} = & \left\{\begin{array}{ll}
    1, & \text{if } i = j \text{ and } i \in \Omega_{\xi},\\
    0, & \text{otherwise.}
  \end{array}\right. .
\end{align*}
Note that from the definition of $f$ in \eqref{eq:deciding-dory}, we obtain the Hessian of $f$ is the expectation of the Hessian of $F$:
\begin{align*}
  \mathbf{H}_f^{\text{emb}} = & \mathbb{E}_{\xi}
  \mathbf{H}_{F_{\xi}}^{\text{emb}} \preccurlyeq  L\mathbb{E}_{\xi} I_{\xi} .
\end{align*}
Finally following the definition of $\alpha$, we obtain $\mathbf{H}_f^{\text{emb}} \preccurlyeq L\mathbb{E}_{\xi} I_{\xi} \preccurlyeq L \alpha$, which leads to the $\alpha L$ Lipschitzian continuity in the sparse weights \eqref{eq:crack-skylark-1}.
\end{proof}

\textbf{Proof to Theorem \ref{thm:9853187e}}

\begin{proof}
  Since $f$ is defined as the average of $F$, from Lipschitzian gradient assumption \eqref{eq:able-snipe},
  it can be derived that $f$ is also Lipschitzian:
  \begin{align}
    \label{eq:able-snipe-1}
    \left\| f' \left(\mathbf{w} \right) - f' \left(\mathbf{w} + \mathbf{\delta}\right) \right\| \leqslant L\left\| \mathbf{\delta} \right\|,\quad \forall \mathbf{w}, \mathbf{\delta} \in \mathbb{R}^{\embDimension + \nnDimension}.
  \end{align}

  From the algorithm update rule \eqref{eq:59f6c90b} and the Lipschitzian property \eqref{eq:able-snipe-1}, we obtain

\begin{align*}
  & f \left(\embWeightIter{t+1}, \nnWeightIter{t+1}\right) - f \left(\embWeightIter{t}, \nnWeightIter{t}\right)\\
  \overset{\eqref{eq:able-snipe-1}}{\leqslant} & \left \langle f' \left(\embWeightIter{t}, \nnWeightIter{t}\right), \embWeightIter{t+1} + \nnWeightIter{t+1} -
  \embWeightIter{t} - \nnWeightIter{t} \right \rangle + \frac{L}{2} \left\| \embWeightIter{t+1} + \nnWeightIter{t+1} - \embWeightIter{t} - \nnWeightIter{t} \right\|^2\\
  \overset{\eqref{eq:59f6c90b}}{=} & - \gamma \left\langle f' \left(\embWeightIter{t}, \nnWeightIter{t}\right),  \embGradientIter{t}  + \nnGradientIter{t} \right\rangle + \frac{L \gamma^2}{2} \left\| \embGradientIter{t} + \nnGradientIter{t} \right\|^2 \\
  = & - \gamma \left \langle f' \left(\embWeightIter{t}, \nnWeightIter{t}\right), F_t'\right \rangle + \frac{L \gamma^2}{2} \left\| F_t' \right\|^2 .
\end{align*}

Taking the expectation on both sides, we obtain
\begin{align*}
  \mathbb{E} \left[f \left(\embWeightIter{t+1}, \nnWeightIter{t+1}\right) - f \left(\embWeightIter{t}, \nnWeightIter{t}\right)\right]
  \leqslant - \gamma \underbrace{\mathbb{E} \left[\left \langle f' \left(\embWeightIter{t}, \nnWeightIter{t}\right),
  F_t' \right \rangle\right]}_{=:T_{1}} + \frac{L \gamma^2}{2} \underbrace{\mathbb{E} \left[ \left\| F_t' \right\|^2\right]}_{=:T_{2}} .\numberthis \label{eq:cosmic-beetle}
\end{align*}

We show the bound for the terms $T_{1}$ and $T_{2}$ respectively, in the following
paragraphs. For $T_{2}$, from the unbiased property of stochastic gradients \eqref{eq:deciding-dory}
and the bounded variance assumption
\eqref{eq:credible-bird} we obtain
\begin{align*}
  T_{2} = & \mathbb{E} \left[\left\| F_t' \right\|^2\right]\\
   = & \mathbb{E} \left[ \left\| F_t' - f' \left(\embWeightIter{\mathD (t)}, \nnWeightIter{t}\right) + f' \left(\embWeightIter{\mathD (t)}, \nnWeightIter{t}\right) \right\|^2\right]\\
  = & \mathbb{E} \left[ \left\| F_t' - f' \left(\embWeightIter{\mathD (t)}, \nnWeightIter{t}\right) \right\|^2
  + \left\| f' \left(\embWeightIter{\mathD (t)}, \nnWeightIter{t}\right) \right\|^2 \right]
  + 2 \underbrace{\mathbb{E} \left[ \left\langle F_t' - f' \left(\embWeightIter{\mathD (t)}, \nnWeightIter{t}\right), f'\left(\embWeightIter{\mathD (t)}, \nnWeightIter{t}\right)
  \right\rangle\right]}_{= 0, \text{ from \eqref{eq:deciding-dory}}}\\
  \overset{\eqref{eq:credible-bird}}{\leqslant} & \sigma^2 +\mathbb{E} \left[ \left\| f' \left(\embWeightIter{\mathD (t)}, \nnWeightIter{t}\right) \right\|^2\right]. \label{eq:trusted-goose}\numberthis
\end{align*}

For $T_{1}$, we have
\begin{align*}
  T_{1} = & \mathbb{E} \left[ \left\langle f' \left(\embWeightIter{t}, \nnWeightIter{t}\right), F_t'\right\rangle \right]\\
  \overset{\eqref{eq:deciding-dory}}{=} & \mathbb{E} \left[\left \langle f' (\embWeightIter{t}, \nnWeightIter{t}), f' \left(\embWeightIter{\mathD (t)}, \nnWeightIter{t}\right) \right \rangle \right]\\
  = & \frac{1}{2} \mathbb{E} \left[ \left\| f' \left(\embWeightIter{t}, \nnWeightIter{t}\right) \right\|^2 + \left\| f'
  \left(\embWeightIter{\mathD (t)}, \nnWeightIter{t}\right) \right\|^2 - \left\| f' \left(\embWeightIter{t}, \nnWeightIter{t}\right) - f'\left
  (\embWeightIter{\mathD (t)}, \nnWeightIter{t}\right) \right\|^2 \right] ,\label{eq:smashing-jackal}\numberthis
\end{align*}
where the last step comes from the fact that for any $\mathbf{a},\mathbf{b}$ the following equality holds:
$\left \langle \mathbf{a}, \mathbf{b}\right \rangle = \frac{1}{2}\left\|\mathbf{a} - \mathbf{b} \right\|^{2} - \frac{1}{2} \left\|\mathbf{a}\right\|^{2} - \frac{1}{2}\left\|\mathbf{b}\right\|^{2}$.

Using the bounds for $T_{1}$ \eqref{eq:smashing-jackal} and $T_{2}$ \eqref{eq:trusted-goose} in \eqref{eq:cosmic-beetle}, we obtain
\begin{align*}
  & \mathbb{E} \left[f \left(\embWeightIter{t+1}, \nnWeightIter{t+1}\right) - f\left(\embWeightIter{t}, \nnWeightIter{t}\right)\right]\\
  \leqslant & - \frac{\gamma}{2} \mathbb{E} \left[ \left\| f' \left(\embWeightIter{t}, \nnWeightIter{t}\right)
  \right\|^2 + \left\| f' \left(\embWeightIter{\mathD (t)}, \nnWeightIter{t}\right) \right\|^2 - \left\| f' \left(\embWeightIter{t},
  \nnWeightIter{t}\right) - f' \left(\embWeightIter{\mathD (t)}, \nnWeightIter{t}]\right) \right\|^2 \right]\\
  & + \frac{L \gamma^2}{2} \mathbb{E}\left[ \left\| f' \left(\embWeightIter{\mathD (t)}, \nnWeightIter{t}\right)
  \right\|^2\right] + \frac{L \gamma^2}{2} \sigma^2\\
  = & - \frac{\gamma}{2} \mathbb{E} \left[ \left\| f' \left(\embWeightIter{t}, \nnWeightIter{t}\right) \right\|^2\right] -
  \frac{\gamma}{2} \left(1 - \gamma L\right) \mathbb{E} \left[ \left\| f' \left(\embWeightIter{\mathD (t)}, \nnWeightIter{t}\right) \right\|^2 \right]
   + \frac{L \gamma^2}{2} \sigma^2 + \frac{\gamma}{2} \mathbb{E} \left[ \left\| f'
  \left(\embWeightIter{t}, \nnWeightIter{t}\right) - f' \left(\embWeightIter{\mathD (t)}, \nnWeightIter{t}\right) \right\|^2 \right]\\
  \leqslant & - \frac{\gamma}{2} \mathbb{E}\left[ \left\| f' \left(\embWeightIter{t}, \nnWeightIter{t}\right) \right\|^2 \right] +
  \frac{L \gamma^2}{2} \sigma^2 + \frac{\gamma}{2} \underbrace{\mathbb{E} \left[\left\| f' \left(\embWeightIter{t},\nnWeightIter{t}\right) - f' \left(\embWeightIter{\mathD (t)}, \nnWeightIter{t}\right) \right\|^2\right]}_{=:\omega_{t}},\numberthis \label{eq:fancy-guinea}
\end{align*}
where the last inequality is obtained by choosing a sufficiently small learning rate $\gamma$ satisfying $1 \geqslant \gamma L$.

We then bound $\omega_{t}$ in order to get the final convergence rate result. 
Denoting
\begin{align*}%
  L_{\text{emb}} := & \min\left\{1, \alpha\right\}L,\\
  \embGradientIter{t}(t) := & \nabla_{\embWeight}f\left(\embWeightIter{\mathD (t)}, \embWeightIter{t}\right),\\ 
	\nnGradientIter{t}(t) := & \nabla_{\nnWeight}f\left(\embWeightIter{\mathD (t)}, \embWeightIter{t}\right),%
\end{align*}
it follows from Lemma \ref{lemma:ajksdf} that
\begin{align*}
  \omega_t = & \mathbb{E} \left[\left\| f' \left(\embWeightIter{t}, \nnWeightIter{t}\right) - f' \left(\embWeightIter{\mathD (t)}, \nnWeightIter{t}\right) \right\|^2\right]\\
  \overset{\eqref{eq:crack-skylark-1}}{\leqslant} & L^2_{\text{emb}} \mathbb{E} \left[ \left\| \embWeightIter{t} - \embWeightIter{\mathD (t)} \right\|^2\right]\\
  \overset{\eqref{eq:59f6c90b}}{=} & \gamma^2 L^2_{\text{emb}} \mathbb{E} \left[ \left\| \sum^{t - 1}_{k = \mathD (t)} \embGradientIter{k}\right\|^2 \right]\\
  = & \gamma^2 L^2_{\text{emb}} \mathbb{E} \left[ \left\| \sum^{t - 1}_{k = \mathD (t)} \embGradientIter{k}- \sum^{t - 1}_{k = \mathD (t)}\embGradientIter{k} + \sum^{t - 1}_{k =
  \mathD (t)}\embGradientIter{k} \right\|^2 \right]\\
  = & \gamma^2 L^2_{\text{emb}} \mathbb{E} \left[ \left\| \sum^{t - 1}_{k = \mathD (t)}
  \embGradientIter{k}- \sum^{t - 1}_{k = \mathD (t)}\embGradientIter{k}
  \right\|^2\right]
  + \gamma^2 L^2_{\text{emb}} \mathbb{E} \left[ \left\| \sum^{t - 1}_{k = \mathD (t)}
 \embGradientIter{k} \right\|^2 \right]\\
  \overset{\eqref{eq:credible-bird}}{\leqslant} & \gamma^2 L^2_{\text{emb}}  \tau \sigma^2 + \gamma^2 L^2_{\text{emb}}
  \mathbb{E} \left[ \left\| \sum^{t - 1}_{k = \mathD (t)}\embGradientIter{k} \right\|^2 \right]\\
  \leqslant & \gamma^2 \tau L^2_{\text{emb}}  \tau \sigma^2 + \gamma^2 L^2_{\text{emb}}
  \mathbb{E} \left[  \sum^{t - 1}_{k = \mathD (t)} \left\|
 \embGradientIter{k} \right\|^2  \right],\numberthis \label{eq:awake-opossum}
\end{align*}
where the third last step and the last step come from the fact that for any
$a_{1}, a_{2}, \ldots a_n$, we have the following inequality:
$\left\|a_{1} + a_{2} + \cdots + a_{n}\right\|^{2} \leqslant n \sum_{i=1}^{n}\left\|a_{i}\right\|^{2}$.
We then bound the right side of \eqref{eq:awake-opossum} with $\omega_t$ itself:
\begin{align*}
  \omega_t \leqslant & \gamma^2 L^2_{\tau} \tau \sigma^2 + \gamma^2 \tau L^2_{\tau} \mathbb{E} \left[ \sum^{t - 1}_{k = \mathD (t)} \left\|\embGradientIter{k} \right\|^2 \right]\\
  = & \gamma^2 L^2_{\tau} \tau \sigma^2 + \gamma^2 \tau L^2_{\tau} \mathbb{E}
  \left[ \sum^{t - 1}_{k = \mathD (t)} \left\|\embGradientIter{k}
    - \nabla_{\embWeight}{f} \left(\embWeightIter{k},
    \nnWeightIter{k} \right) + \nabla_{\embWeight}{f}
  \left(\embWeightIter{k}, \nnWeightIter{k}\right) \right\|^2
  \right]\\
  \leqslant & \gamma^2 L^2_{\tau} \tau \sigma^2 + 2 \gamma^2 \tau L^2_{\tau}
  \mathbb{E} \left[ \sum^{t - 1}_{k = \mathD (t)} \left\|\embGradientIter{k} - \nabla_{\embWeight}f
  \left(\embWeightIter{k}, \nnWeightIter{k}\right) \right\|^2 +
  \sum^{t - 1}_{k = \mathD (t)} \left\| \nabla_{\embWeight}f
  \left(\embWeightIter{k}, \nnWeightIter{k}\right) \right\|^2
  \right]\\
  \leqslant & \gamma^2 L^2_{\tau} \tau \sigma^2 + 2 \gamma^2 \tau L^2_{\tau}
  \mathbb{E} \left[ \sum^{t - 1}_{k = \mathD (t)} \omega_k + \sum^{t - 1}_{k = \mathD (t)} \left\| f' \left(\mathbf{w}_k\right) \right\|^2 \right]\\
  \leqslant & \gamma^2 L^2_{\tau} \tau \sigma^2 + 2 \gamma^2 \tau L^2_{\tau}
  \mathbb{E} \left[ \sum^{t - 1}_{k = t - \tau} \omega_k + \sum^{t - 1}_{k = t - \tau} \left\| f' \left(\mathbf{w}_k\right) \right\|^2 \right] .\numberthis \label{eq:up-squid}
\end{align*}

Summing \eqref{eq:up-squid} from $t=0$ to $t = T-1$ we obtain
\begin{align*}
  \sum_{t = 0}^{T - 1} \omega_t \leqslant & \gamma^2 L^2_{\tau} \tau \sigma^2
  T + 2 \gamma^2 \tau^2 L^2_{\tau} \mathbb{E} \left[ \sum^{T - 1}_{t = 0}
  \omega_k + \sum^{T - 1}_{t = 0} \left\| f' \left(\mathbf{w}_t\right) \right\|^2 \right] .
\end{align*}

Rearranging the terms, we obtain
\begin{align*}
  \left(1 - 2 \gamma^2 \tau^2 L^2_{\text{emb}}\right) \sum_{t = 0}^{T - 1} \omega_t \leqslant &
  \gamma^2 L^2_{\text{emb}} \tau \sigma^2 T + 2 \gamma^2 \tau^2 L^2_{\text{emb}}
  \mathbb{E} \left[ \sum^{T - 1}_{t = 0} \left\| f' \left(\mathbf{w}_t\right) \right\|^2 \right]
  .
\end{align*}
By choosing a sufficiently small learning rate $\gamma$ satisfying $1/2 \geqslant \gamma \tau L_{\text{emb}}$, it follows that
\begin{align*}
  \sum_{t = 0}^{T - 1} \omega_t \leqslant & 2 \gamma^2 L^2_{\text{emb}} \tau
  \sigma^2 T + 4 \gamma^2 \tau^2 L^2_{\text{emb}} \mathbb{E} \left[ \sum^{T - 1}_{t
  = 0} \left\| f' \left(\mathbf{w}_t\right)\right\|^2 \right] . \numberthis \label{eq:organic-snapper}
\end{align*}

Now we get back to \eqref{eq:fancy-guinea}. Summing \eqref{eq:fancy-guinea} from $t=0$ to $t=T-1$ we obtain
\begin{align*}
  & \mathbb{E} \left[f \left(\embWeightIter{t}, \nnWeightIter{t}\right) - f \left(\embWeightIter{0}, \nnWeightIter{0}\right)\right]\\
  \leqslant & - \frac{\gamma}{2} \sum_{t = 0}^{T - 1} \mathbb{E} \left[ \left\| f'
  \left(\embWeightIter{t}, \nnWeightIter{t}\right) \right\|^2 \right]+ \frac{L \gamma^2}{2} \sigma^2 T + \frac{\gamma}{2}\sum_{t=0}^{T-1} \omega_t \\
  \overset{\eqref{eq:organic-snapper}}{\leqslant} & - \frac{\gamma}{2} \sum_{t = 0}^{T - 1} \mathbb{E} \left[ \left\| f'\left(\embWeightIter{t}, \nnWeightIter{t}\right) \right\|^2\right] + \frac{L \gamma^2}{2} \sigma^2 T + \gamma^3 L^2_{\text{emb}} \tau
  \sigma^2 T + 2 \gamma^3 \tau^2 L^2_{\text{emb}} \mathbb{E} \left[ \sum^{T - 1}_{t
  = 0} \left\| f' \left(\embWeightIter{t}, \nnWeightIter{t}\right) \right\|^2 \right] \\
  \leqslant & - \frac{\gamma}{4} \sum_{t = 0}^{T - 1} \mathbb{E} \left[ \left\| f'\left(\embWeightIter{t}, \nnWeightIter{t}\right) \right\|^2\right] + \frac{L \gamma^2}{2} \sigma^2 T + \gamma^3 L^2_{\text{emb}} \tau
  \sigma^2 T, 
\end{align*}
where the last step comes from choosing a sufficiently small learning rate $\gamma$ satisfying $1/4 \geqslant \gamma \tau L_{\text{emb}}$.
Rearranging the terms we obtain:
\begin{align*}
  & \frac{\sum_{t = 0}^{T - 1} \mathbb{E} \left[ \left\| f' \left(\embWeightIter{t}, \nnWeightIter{t}\right) \right\|^2\right]}{T}\\
  \leqslant & \frac{4\mathbb{E} \left[f \left(\embWeightIter{0}, \nnWeightIter{0}\right) - f \left(\embWeightIter{t},\nnWeightIter{t}\right)\right]}{\gamma T} + 2 L \gamma \sigma^2 + 4 \gamma^2 L^2_{\text{emb}}\tau \sigma^2 \\
    \leqslant & \frac{4\mathbb{E} \left[f \left(\embWeightIter{0}, \nnWeightIter{0}\right) - f \left(\embWeightIter{t}, \nnWeightIter{t}\right)\right]}{\gamma T} + 2 L \gamma \sigma^2 + \gamma L_{\text{emb}} \sigma^2 \\
      \leqslant & \frac{4\mathbb{E} \left[f \left(\embWeightIter{0}, \nnWeightIter{0}\right) - f^*\right]}{\gamma T} + 3L \gamma \sigma^2,
\end{align*}
where the last step comes from the definition of $L_{\text{emb}}$ and the second last step still comes from the our choosing a sufficiently small learning rate $\gamma$ satisfying $1/4 \geqslant \gamma \tau L_{\text{emb}}$.

By choosing the learning rate $\gamma$ to be the same as the optimal learning rate for SGD $\frac{1}{L + \sqrt{TL} \sigma + 4\tau L_{\text{emb}}}$, we obtain
\begin{align*}
  & \frac{\sum_{t = 0}^{T - 1} \mathbb{E} \left[ \left\| f' \left(\embWeightIter{t}, \nnWeightIter{t}\right) \right\|^2\right]}{T}\\
  \lesssim & \frac{ \sqrt{L} \sigma}{\sqrt{T}} + \frac{L}{T} + \frac{\tau L_{\text{emb}}}{T} \\
  \leqslant & \frac{ \sqrt{L} \sigma}{\sqrt{T}} + \frac{L}{T} + \frac{\tau \min\left\{1, \alpha\right\}}{T}
  \nonumber \\
\end{align*}
completing the proof.
\end{proof}
%
%

\end{document}